\documentclass{article}
\usepackage{iclr2027_conference}
\usepackage[T1]{fontenc}
\usepackage{newtxtext}

\usepackage{hyperref}
\usepackage{graphicx}
\usepackage{adjustbox}

\usepackage{caption}
\usepackage{subcaption}
\usepackage{multirow}
\usepackage{tcolorbox}
\tcbuselibrary{skins,breakable,listings}
\usepackage{listings}
\usepackage{xcolor}
\usepackage{wrapfig}   

\usepackage{algorithm}
\usepackage{algorithmic}
\usepackage{amsthm}
\usepackage{bbm}
\usepackage{amssymb}
\usepackage{amsmath}
\newtheorem{theorem}{Theorem}
\newtheorem{definition}{Definition}
\newtheorem{assumption}{Assumption}

\newtheorem{proposition}{Proposition}

\renewcommand{\thesection}{\arabic{section}}
\renewcommand{\thesubsection}{\thesection.\arabic{subsection}}
\usepackage{newfloat}
\usepackage{listings}
\DeclareCaptionStyle{ruled}{labelfont=normalfont,labelsep=colon,strut=off} 
\floatstyle{ruled}
\newfloat{listing}{tb}{lst}{}
\floatname{listing}{Listing}

\usepackage{booktabs}
\title{RoMeRL: Balancing Feedback Coverage and the Memory-Reward Trap
in Self-Evolving Agent Memory via Reduced-Order Utility States
}
\author{%
\parbox{0.98\textwidth}{%
\centering
{\normalfont\bfseries\small
Yi Yang\textsuperscript{1,*},
Zhennan Chen\textsuperscript{1,*,\ensuremath{\dagger}},
Yihong Zhuang\textsuperscript{2},
Tiehan Fan\textsuperscript{1},\\[2pt]
Yinan Chen\textsuperscript{3},
Jian Li\textsuperscript{1},
Jian Yang\textsuperscript{1},
Ying Tai\textsuperscript{1,\ensuremath{\ddagger}}
}\\[7pt]
{\normalfont\small
\textsuperscript{1}Nanjing University \quad
\textsuperscript{2}Xiamen University \quad
\textsuperscript{3}Zhejiang University
}
}}

\hypersetup{
  pdftitle={RoMeRL: Balancing Feedback Coverage and the Memory-Reward Trap in Self-Evolving Agent Memory via Reduced-Order Utility States},
  pdfauthor={Yi Yang, Zhennan Chen, Yihong Zhuang, Tiehan Fan, Yinan Chen, Jian Li, Jian Yang, Ying Tai}
}

\iclrfinalcopy

\begin{document}

\maketitle
\lhead{Preprint}
\begingroup
\renewcommand{\thefootnote}{\fnsymbol{footnote}}
\footnotetext[1]{Equal contribution.}
\footnotetext[2]{Project Leader.}
\footnotetext[3]{Corresponding Author.}
\endgroup

\begin{abstract}
Learning-based memory systems for self-evolving LLM agents face two
tightly coupled challenges. First, trajectory-indexed utilities grow
with the interaction history, thereby dispersing limited feedback over
an ever-expanding state space. Second, because trajectory-level rewards
are jointly assigned to co-retrieved memories, irrelevant experiences
may receive misleading utility updates and consequently enter the
\emph{memory-reward trap}. To address these challenges, we introduce
Reduced-Order Memory Reinforcement Learning (RoMeRL), which represents
the growing trajectory-indexed utility space using a fixed-dimensional
per-task memory state factorized by outcome polarity and memory
dynamics. RoMeRL incorporates new experiences through a fixed set of semantic
coordinates whose contents are updated or replaced over time, thereby
concentrating feedback over a bounded utility support. Theoretically, we show that this reduced-order parameterization increases
the average feedback received by each utility coordinate and characterize
the steady-state occupancy of erroneous coordinates under a generic
coordinate-transition model. Empirically, across LifelongAgentBench, ALFWorld, and AppWorld,
RoMeRL achieves the highest overall average score, reduces the Cold-Q ratio by
$80.0\%$, increases feedback density by approximately $6.0\times$,
reduces the maintained memory size by $84.4\%$, and cuts LLM calls by
$21.1\%$.  These results show that reduced-order utility states support efficient
self-evolving agent memory while limiting persistent reward
contamination. Code is available at: \url{https://github.com/YOUNG-fnxm/RoMeRL}
\end{abstract}

\section{Introduction}
Large language model (LLM) agents are inherently stateless, which limits their ability to accumulate and reuse experience across interactions \citep{sumers2023cognitive, tao2024survey}. Agent memory addresses this limitation by storing and retrieving past experience; recent surveys organize this design space by memory substrate,
cognitive role, operational lifecycle, and externalized agent infrastructure
\citep{zhang2025survey,huang2026rethinking,zhou2026externalization}. Early methods maintain episodic memories or reusable skill libraries and retrieve trajectories by semantic similarity \citep{zhong2024memorybank, cai2025flex, wang2023voyager}. Later systems introduce explicit memory lifecycles, using handcrafted workflows, reflection, summarization, and rule-based selection to organize and refine agent experience \citep{packer2023memgpt, fu2024autoguide, ouyang2025reasoningbank, zhao2024expel, wang2024agent}. Recent learning-based approaches instead optimize memory generation, retrieval, and utility from downstream task outcomes, allowing the memory system to evolve through interaction without updating the underlying LLM \citep{yan2026memory, zhou2025memento, zhang2025memevolve, zhang2026memrl}. The focus has consequently shifted from preserving past experience to deciding which experiences should remain active and affect future behavior.

Existing end-to-end memory optimization methods commonly assign a separate utility to every stored trajectory and update it from downstream task outcomes. As experience accumulates, this trajectory-indexed formulation continually expands the dimensionality of the learnable memory state while feedback remains limited, resulting in widespread utility cold start, concentrated updates, and low feedback density. A natural remedy is to increase exploration so that under-visited memories receive more feedback. However, our experiments with Upper Confidence Bound (UCB) show that, although stronger exploration improves memory coverage and alleviates cold start, it degrades task performance. Because trajectory-level rewards are jointly assigned to co-retrieved memories, broader exploration also exposes more weakly relevant memories to successful contexts, allowing them to receive positive updates without corresponding contributions. We refer to this exploration--contamination dilemma as the \textit{Memory-Reward Trap (MRT)}, illustrated in Figure~\ref{fig:memory_reward_trap1}.

\begin{wrapfigure}{r}{0.48\textwidth}   
    \centering
    \includegraphics[width=\linewidth]{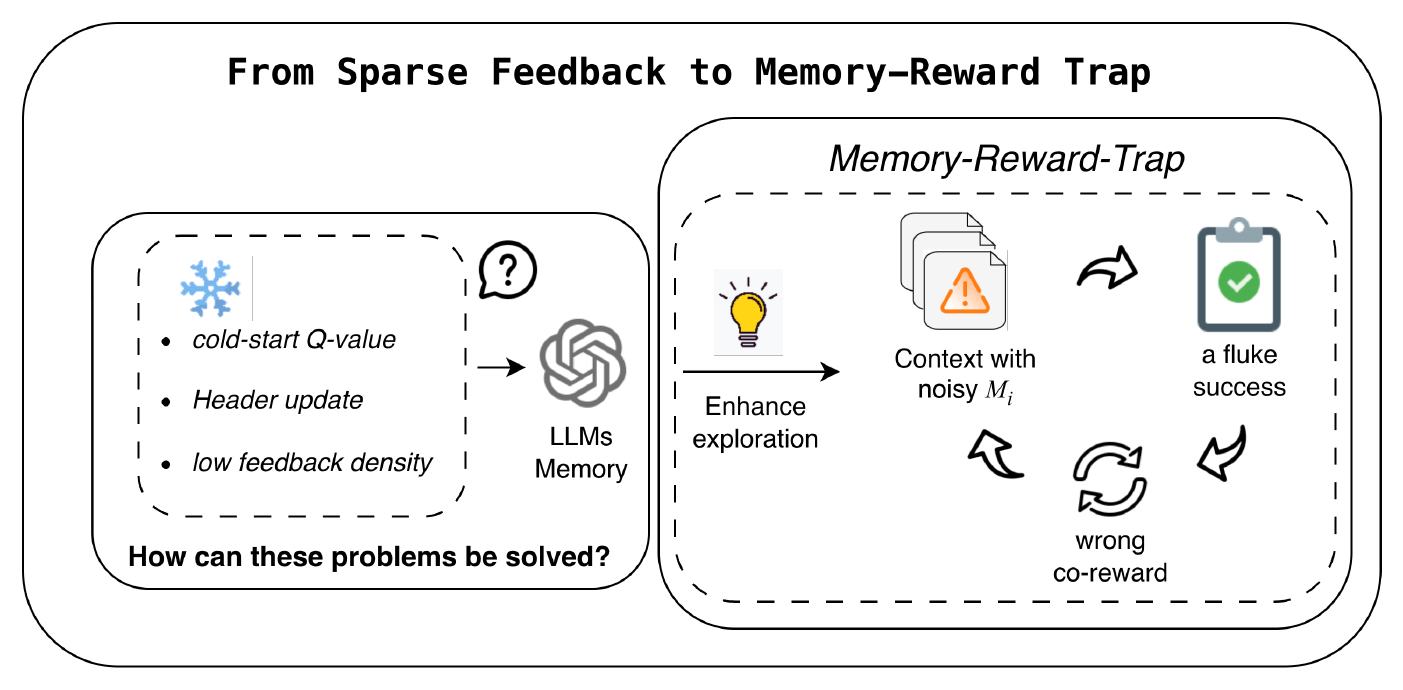}
    \caption{\textbf{\textit{Memory-reward trap}}: Stronger exploration retrieves more low-access-count memories, including weakly relevant, noisy, outdated, or failed experiences. When such a memory appears in a successful episode, the shared trajectory reward can promote it despite no causal contribution. Exploration therefore increases feedback coverage but also raises the risk of reward misattribution.}
    \label{fig:memory_reward_trap1}
\end{wrapfigure}

This conflict raises a more fundamental question:
\textbf{\textit{How can an agent improve memory-feedback coverage
without expanding the utility support exposed to erroneous reward
propagation?}}
We trace this dilemma to trajectory-indexed utility learning: every newly
stored experience introduces an additional utility variable, while
broader exploration exposes more such variables to potentially
misattributed rewards. To address this problem, we propose
Reduced-Order Memory Reinforcement Learning
(RoMeRL). Instead of exploring a continually expanding utility space more
aggressively, RoMeRL changes the state on which memory reinforcement
learning operates. It replaces trajectory-indexed utilities with a
fixed-dimensional per-task state factorized by outcome polarity and memory
dynamics. Outcome polarity separates positive from negative evidence;
memory dynamics separates consolidated historical evidence from adaptive
state-transition evidence. For each task, their Cartesian product defines a small set
of persistent semantic coordinates whose contents are updated or
replaced as new experiences arrive. The resulting state concentrates
feedback while limiting the dimensionality and persistence of support
exposed to the memory-reward trap.

We evaluate RoMeRL on LifelongAgentBench
\citep{zheng2025lifelongagentbench}, ALFWorld
\citep{shridhar2020alfworld}, and AppWorld
\citep{trivedi2024appworld}, which jointly cover operating-system and
database interaction, embodied household planning, and compositional
multi-application workflows. Across the ten evaluation units reported in
Table~\ref{tab:runtime_main_results}, RoMeRL achieves the highest overall
average score of $0.753$, outperforming the strongest baseline evaluated on
all three benchmarks by $2.9$ percentage points. On AppWorld, RoMeRL improves
SGC from $0.286$ to $0.326$ while remaining competitive in TGC
($0.306$ versus $0.313$). It also reduces the Cold-Q ratio by $80.0\%$,
increases feedback density by approximately $6.0\times$, decreases the
maintained memory size by $84.4\%$, and reduces LLM calls by $21.1\%$ without
updating the underlying LLM.

Our contributions are as follows:
\begin{itemize}
    \item We show that trajectory-indexed memory learning dilutes
    feedback as the utility state grows and expands the persistent
    support exposed to the memory-reward trap.

    \item We propose RoMeRL, which replaces each task's growing memory utility
    space with a fixed-dimensional state factorized by outcome polarity
    and memory dynamics.

    \item We characterize the balance between feedback and the memory-reward trap
      in reduced-order utility states, and evaluate RoMeRL's effectiveness
      and efficiency across terminal, embodied, and multi-application agent
      benchmarks.
\end{itemize}
\section{Related Work}
\paragraph{Workflow-Based Agent Memory.}
These systems build on retrieval-augmented inference, where external
evidence is selected by learned relevance models and supplied as model
context \citep{karpukhin2020dense,lewis2020retrieval}. Early agent-memory systems externalize interaction history through
predefined storage, reflection, and retrieval workflows. Generative
Agents and MemoryBank maintain episodic streams with importance-based
retrieval, reflection, and forgetting
mechanisms~\citep{park2023generative,zhong2024memorybank}. Reflexion and
ExpeL convert task feedback into verbal reflections or transferable
insights~\citep{shinn2023reflexion,zhao2024expel}, while Voyager and
AutoGuide distill interactions into reusable skills or
context-dependent guidelines~\citep{wang2023voyager,fu2024autoguide}.
MemGPT and HiAgent provide virtual or hierarchical memory
management~\citep{packer2023memgpt,hu2025hiagent}; A-MEM links structured
notes into an evolving knowledge network~\citep{xu2026amem}. These
methods avoid model tuning, but their memory decisions are largely
governed by handcrafted lifecycles, prompting rules, and semantic
similarity rather than downstream task outcomes.
\paragraph{Learning-Based Agent Memory.}
Recent work treats agent memory as a learnable component rather than a
static retrieval store. ReasoningBank and MemP distill trajectories into
evolving reasoning or procedural memories~\citep{ouyang2025reasoningbank,fang2026memp}.
Memory-R1, Agentic Memory, and AtomMem use reinforcement learning to
train memory construction, retrieval, update, and deletion
policies~\citep{yan2026memory,yu2026agentic,yao2026atommem}. Fine-Mem
assigns fine-grained rewards to individual memory operations
\citep{ma2026fine}, whereas MemEvolve adapts the memory architecture
across tasks~\citep{zhang2025memevolve}. MemRL, the closest setting to
ours, learns episodic-memory utilities through non-parametric runtime
reinforcement learning~\citep{zhang2026memrl}. These methods generally
optimize a growing collection of memory entries; RoMeRL
instead represents each task's utility-bearing state with a fixed number of
semantic coordinates.

\section{Preliminaries}

\paragraph{Reinforcement Learning in Agentic Memory.}

Outcome-driven agent memory methods associate each stored trajectory $m_i$ with a learnable utility $Q_{i,t}$ and use the utility together with semantic relevance for memory retrieval~\cite{salama2025meminsight,zhang2025learn,zhang2026memrl}. Let
$
\mathcal{M}_t=\{(m_i,Q_{i,t})\}_{i=1}^{N_t}
$
denote the memory bank at interaction step $t$, and let $\mathcal{S}_t\subseteq\mathcal{M}_t$ be the memories retrieved for the current query. After the agent completes the trajectory and receives task-level reward $r_t$, the utility of each retrieved memory is updated as
\begin{equation}
Q_{i,t+1}
=
Q_{i,t}
+
\alpha\,\mathbb{I}[m_i\in\mathcal{S}_t]
\bigl(r_t-Q_{i,t}\bigr),
\label{eq:memory_rl_update}
\end{equation}
where $\alpha$ is the learning rate. This formulation improves the agent by updating external memory utilities rather than the parameters of the underlying LLM. However, because every newly stored trajectory introduces an additional utility variable, the learnable memory state
$
\mathbf{Q}_t=(Q_{1,t},\ldots,Q_{N_t,t})\in\mathbb{R}^{N_t}
$
grows continuously with the interaction history.

\section{Reduced-Order Memory RL}
In this section, we formulate end-to-end memory reinforcement learning as
utility learning over a growing trajectory-indexed state. Bundle-level
rewards estimate observational rather than marginal utility, giving rise to
the \emph{memory-reward trap}, while reliable full-state estimation requires
feedback that grows with the number of stored trajectories. Together, these
properties create an exploration dilemma: broader coverage exposes more
utility variables to misattributed rewards. RoMeRL addresses both
problems by replacing each task's growing index set with a fixed-dimensional
state factorized by outcome polarity and memory dynamics. We first analyze
feedback concentration and erroneous-coordinate occupancy for a general
active dimension, and then present the practical implementation. The overall architecture of RoMeRL is shown in Figure~\ref{fig:ablation_combined}.
\begin{figure*}[htb]
    \centering
    \setlength{\abovecaptionskip}{6pt}
	\includegraphics[width=0.85\textwidth]{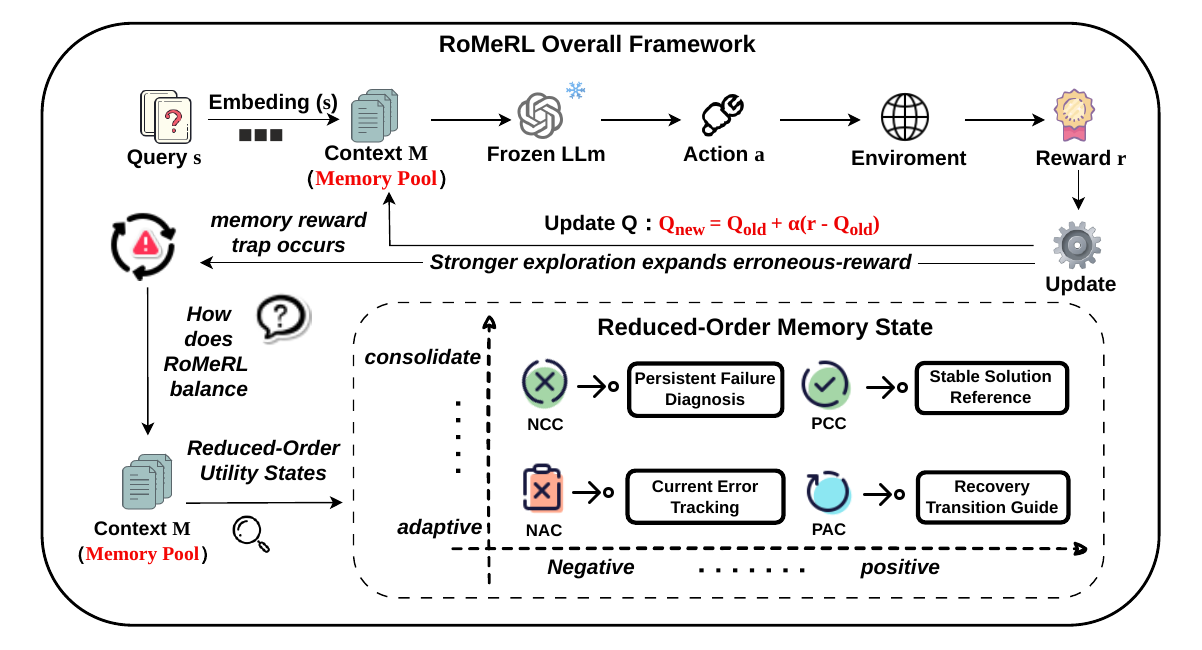}
    \caption{Overview of RoMeRL. \textbf{Top:} the agent retrieves memories to construct the LLM context and updates memory utilities using task-level rewards. Although stronger exploration improves feedback coverage, it also expands the utility support exposed to erroneous reward propagation, increasing the risk of the memory-reward trap.
\textbf{Bottom:} RoMeRL maps each task's growing trajectory-indexed utility space into a fixed-dimensional state factorized by outcome polarity and memory dynamics. Its four semantic coordinates PCC, PAC, NCC, and NAC are updated through
online retention, promotion, and replacement, thereby concentrating
feedback over a bounded active support.
}
    \label{fig:ablation_combined}
\end{figure*}
\subsection{Trajectory-Indexed Feedback and the Memory-Reward Trap}

Under the trajectory-indexed parameterization, each stored trajectory $m_i$ corresponds to an independently learned utility $Q_i$, whereas the task-level reward is jointly determined by the complete retrieved context $\mathcal{S}_t$. To distinguish the utility estimated from bundle-level outcomes from the individual contribution of a memory, we introduce the following definition.

\begin{definition}[Trajectory-Indexed Memory Utility]
\label{def:trajectory_indexed_utility}
Let
$
\small
n_i
=
\sum_t \mathbbm{1}[m_i\in\mathcal{S}_t],
\hat{Q}_i
=
\frac{1}{n_i}
\sum_{t:m_i\in\mathcal{S}_t}R_t,
\quad n_i>0,
$
denote the retrieval count and empirical utility of memory $m_i$.
Its observational utility is
$
\small
\mu_i
=
\mathbb{E}[R_t\mid m_i\in\mathcal{S}_t].
\label{eq:observational_utility}
$
Let the interventional outcomes with and without $m_i$ be
$
v_i^1
=
\mathbb{E}[R_t\mid\operatorname{do}(m_i\in\mathcal{S}_t)],
v_i^0
=
\mathbb{E}[R_t\mid\operatorname{do}(m_i\notin\mathcal{S}_t)],
$
where the task distribution and remaining retrieval context are fixed.
The marginal utility, observational attribution bias, and total credit gap
are respectively
\begin{equation}
\small
\begin{gathered}
\theta_i = v_i^1 - v_i^0,
a_i = \mu_i - v_i^1,\\
\mu_i - \theta_i = v_i^0 + a_i.
\end{gathered}
\label{eq:memory_feedback_bias}
\end{equation}
Here, $v_i^0$ is the task-level baseline, while $a_i$ captures attribution
induced by retrieval selection, co-retrieved memories, context interactions,
and task-level confounders.
\end{definition}

Because only bundle-level rewards are observed, $\hat{Q}_i$ estimates the
raw-return target $\mu_i$, rather than the marginal contribution $\theta_i$.
Their gap consists of the task-level baseline $v_i^0$ and the observational
attribution bias $a_i$.

We next characterize how the credit gap and the number of observations
jointly determine the accuracy of trajectory-indexed utility estimation.

\begin{theorem}[Gap--Variance Decomposition of Memory Utility]
\label{thm:bias_variance_memory}
Suppose that the rewards observed when $m_i$ is retrieved are
conditionally independent samples with mean $\mu_i$ and variance
$\sigma_i^2$:
$
\small
\mathbb{E}[R_t\mid m_i\in\mathcal{S}_t]
=
\mu_i,
\operatorname{Var}(R_t\mid m_i\in\mathcal{S}_t)
=
\sigma_i^2.
$
Then
\begin{equation}
\small
\begin{gathered}
\mathbb{E}\bigl[(\hat{Q}_i-v_i^1)^2\bigr] = a_i^2 + \frac{\sigma_i^2}{n_i},\\
\mathbb{E}\bigl[(\hat{Q}_i-\theta_i)^2\bigr] = (v_i^0+a_i)^2 + \frac{\sigma_i^2}{n_i}.
\end{gathered}
\label{eq:memory_bias_variance}
\end{equation}
\end{theorem}

\textit{Proof.}
See Appendix \ref{app:memory_proof}.

Theorem~\ref{thm:bias_variance_memory} separates statistical uncertainty
from two distinct sources of credit mismatch. Retrieving $m_i$ more often
increases $n_i$ and reduces the variance term, but it does not remove the
task-level baseline $v_i^0$ or the observational attribution bias $a_i$.
Additional feedback therefore makes $Q_i$ estimate the raw-return target
$\mu_i$ more precisely; it does not turn the raw-return estimate into the
memory's marginal contribution $\theta_i$.

The most harmful case is the \textbf{\textit{Memory-Reward Trap (MRT)}}:
a memory with non-positive marginal contribution exhibits a positive
observational signal relative to the matched no-memory baseline.

\begin{definition}[Memory-Reward Trap]
\label{def:memory_reward_trap}
A memory $m_i$ falls into a \emph{memory-reward trap} if its
interventional marginal utility is non-positive while its observational
raw return exceeds the interventional baseline without that memory:
\begin{equation}
\small
\theta_i
\leq
0,
\qquad
\mu_i-v_i^0
=
\theta_i+a_i
>
0.
\label{eq:memory_reward_trap}
\end{equation}
Thus, a memory with no positive marginal contribution may still have an
observational utility above the matched baseline due to retrieval selection,
co-retrieved memories, or context interactions. The Q-value estimates
$\mu_i$, while $\mu_i-v_i^0$ represents its observational excess over the
baseline.
\end{definition}

For the exponential update
$
\small
Q_{i,t+1}
=
(1-\alpha)Q_{i,t}
+
\alpha R_t,
\label{eq:ema_memory_update}
$
the expected change after retrieving $m_i$ is
$
\mathbb{E}
\left[
Q_{i,t+1}-Q_{i,t}
\mid
m_i\in\mathcal{S}_t
\right]
=
\alpha(\mu_i-Q_{i,t})
=
\alpha(v_i^0-Q_{i,t})
+
\alpha(\theta_i+a_i).
$
The two terms represent baseline drift and the observational signal above
that baseline, respectively. Under the MRT,
$\theta_i\leq0<\theta_i+a_i$, so repeated Q-based retrieval or retention
can reinforce a memory with non-positive marginal utility.

Let
$
\mathcal{B}_t
=
\left\{
m_i\in\mathcal{M}_t:
\theta_i\leq0,\ 
\mu_i-v_i^0>0
\right\},
\rho_t
=
\Pr(m_i\in\mathcal{B}_t).
$
Then
$
\mathbb{E}[|\mathcal{B}_t|]
=
N_t\rho_t,
$
so the expected support exposed to erroneous reward propagation grows
with the memory-state dimension $N_t$.

\section{Reduced-Order Memory State}
\label{sec:factorized_memory_state}

The preceding analysis separates raw-return estimation from marginal
credit assignment. We now consider a complementary question: how much
feedback is required to estimate raw-return utilities over a growing
trajectory-indexed state? As each stored trajectory introduces an additional utility variable, the dimensionality of this state grows continuously with the interaction history.
\begin{assumption}[Full-Pool Raw-Return Estimation Setting]
\label{assump:full_pool_estimation}
To isolate the effect of utility-state dimensionality, assume that each
stored memory has a stationary raw-return target $\mu_i$. The goal is to
estimate every stored target within error $\epsilon$ with probability at
least $1-\delta$:
\begin{equation}
\small
\Pr
\left(
\forall i\in[N_t],\,
|\hat{Q}_i-\mu_i|\leq\epsilon
\right)
\geq1-\delta.
\label{eq:uniform_full_pool_estimation}
\end{equation}
Let $F_T=\sum_{i=1}^{N_t}n_i$ be the total number of memory-level
updates. If each trajectory updates at most $k$ memories, then
$F_T\leq kT$.
\end{assumption}

This setting characterizes the feedback required to estimate a growing
collection of raw-return utilities. Their relation to the corresponding
marginal contributions is captured by the credit gap $v_i^0+a_i$.

\begin{theorem}[Sufficient Feedback Budget for Trajectory-Indexed Utilities]
\label{thm:full_pool_sample_complexity}
Suppose that rewards are bounded in $[0,1]$ and that each $\hat Q_i$ is estimated from $n_i$ independent samples
with mean $\mu_i$. By
Hoeffding's inequality and a union bound,
Equation~\eqref{eq:uniform_full_pool_estimation} holds if
$
n_i
\geq
\frac{1}{2\epsilon^2}
\log\frac{2N_t}{\delta}, \forall i\in[N_t].
$
Consequently, sufficient feedback and trajectory budgets scale as
\begin{equation}
\small
F_T
=
O\left(
\frac{N_t}{\epsilon^2}
\log\frac{N_t}{\delta}
\right),
\qquad
T
=
O\left(
\frac{N_t}{k\epsilon^2}
\log\frac{N_t}{\delta}
\right).
\label{eq:full_pool_feedback_complexity}
\end{equation}
\end{theorem}

\noindent
\textit{Proof.}
See Appendix \ref{app:memory_proof}.

\paragraph{Relation to marginal contribution.}
Theorem~\ref{thm:full_pool_sample_complexity} concerns estimation of the
raw-return targets $\mu_i$. Relative to the marginal contributions
$
\small
|\hat Q_i-\theta_i|
\leq
|\hat Q_i-\mu_i|
+
|v_i^0+a_i|.
$
Let
$
G_t
=
\max_{i\in[N_t]}|v_i^0+a_i|
$
denote the maximum total credit gap. When $G_t<\epsilon$, the same
concentration argument gives a sufficient feedback budget after replacing
$\epsilon$ with $\epsilon-G_t$. Importantly, $G_t$ contains both the
task-level baseline $v_i^0$ and the observational attribution bias $a_i$;
additional feedback reduces estimation uncertainty around $\mu_i$ but
does not remove either component.

Theorem~\ref{thm:full_pool_sample_complexity} isolates the
dimension-dependent cost of full-pool estimation, while the credit-gap relation shows that broader exploration cannot remove the
task baseline or observational attribution component.

We therefore replace the $N_t$-dimensional trajectory-indexed utility space
with a fixed $2\times2$ state factorized by outcome polarity and memory
dynamics:
\begin{equation}
\small
\mathcal{O}
=
\{+,-\},
\qquad
\mathcal{D}
=
\{\mathrm{C},\mathrm{A}\},
\end{equation}
Here, $+$ and $-$ denote successful and failed evidence, while $\mathrm{C}$
and $\mathrm{A}$ denote consolidated and adaptive dynamics. Consolidated
coordinates retain globally selected evidence; adaptive coordinates track
the current state or meaningful transitions. Their Cartesian product gives
\begin{equation}
\small
\mathcal{I}^{\mathrm{fact}}
=
\mathcal{O}
\times
\mathcal{D},
\qquad
|\mathcal{I}^{\mathrm{fact}}|
=
4.
\label{eq:factorized_index_set}
\end{equation}
For task $g$, the variable-length history $D_{g,t}$ is mapped to
\begin{equation}
\small
\mathbf{Z}_{g,t}
=
\Phi(D_{g,t})
=
\left[
z_{g,t}^{o,d}
\right]_
{(o,d)\in\mathcal{O}\times\mathcal{D}}
=
\begin{bmatrix}
z_{g,t}^{+,\mathrm{C}}
&
z_{g,t}^{+,\mathrm{A}}
\\
z_{g,t}^{-,\mathrm{C}}
&
z_{g,t}^{-,\mathrm{A}}
\end{bmatrix}.
\label{eq:factorized_memory_state}
\end{equation}

Coordinate contents may change, but their semantic identities remain fixed. Each
new trajectory therefore updates or replaces a coordinate rather than adding
a persistent utility variable:
\begin{equation}
\small
\mathbf{Z}_{g,t+1}
=
\mathcal{U}
\left(
\mathbf{Z}_{g,t},
m_{t+1},
R_{t+1}
\right),
\qquad
\dim(\mathbf{Z}_{g,t})=4.
\label{eq:recursive_factorized_update}
\end{equation}

The $2\times2$ state is the smallest complete product of these binary
distinctions. The following results hold for a general active dimension $d$,
while $d=4$ is the design induced by this factorization. We next analyze its
feedback allocation.

\begin{theorem}[Feedback Concentration under State Reduction]
\label{thm:factorized_feedback_density}
Consider a utility-bearing memory state with dimension $d$. Under interaction budget $T$, per-trajectory update limit $k$, and approximately balanced feedback allocation, each utility coordinate receives
$
\bar{n}_d
\approx
\frac{kT}{d}
\label{eq:average_feedback_dimension}
$
feedback signals on average. Therefore, trajectory-indexed learning with dimension $N_t$ and factorized learning with dimension $4$ satisfy
$
\bar{n}_{\mathrm{fact}}
=
\frac{N_t}{4}
\bar{n}_{\mathrm{full}}.
\label{eq:factorized_feedback_ratio}
$
If the reward variance is bounded by $\sigma^2$, the corresponding average utility-estimation variance changes:
\begin{equation}
\small
\text{from}\quad
O
\left(
\frac{\sigma^2N_t}{kT}
\right)
\quad\text{to}\quad
O
\left(
\frac{4\sigma^2}{kT}
\right).
\label{eq:factorized_variance_reduction}
\end{equation}
\end{theorem}

\noindent
\textit{Proof.}
See Appendix \ref{app:memory_proof}.

Theorem~\ref{thm:factorized_feedback_density} shows that state reduction
concentrates a fixed feedback budget over fewer utilities, whereas
full-pool exploration only redistributes feedback across a growing set.

A smaller active state also limits simultaneous exposure to misleading
feedback, while its persistence depends on the contamination and
replacement dynamics characterized by the following proposition.
\begin{proposition}[Erroneous-Coordinate Occupancy under a Transition Model]
\label{thm:factorized_reward_exposure}
Let
$
\rho_{\mathrm{full}}
=
\Pr(\theta_i\leq0,\ \mu_i-v_i^0>0)
$
be the trap probability of a trajectory-indexed utility, so that the
expected number of trap-affected utilities in a full pool of size $N_t$
is $N_t\rho_{\mathrm{full}}$. Consider a generic active state of dimension
$d$ in which each coordinate follows stationary clean--erroneous
transitions. Suppose that the clean-to-erroneous transition probability is
at most $\gamma$ and the erroneous-to-clean transition probability is at
least $\lambda>0$. Then the steady-state erroneous fraction is at most
$\gamma/(\gamma+\lambda)$, and the expected number of erroneous active
coordinates is at most
$
\small
d\frac{\gamma}{\gamma+\lambda}.
\label{eq:factorized_active_exposure_bound}
$
This occupancy is lower than the expected full-pool exposure whenever
\begin{equation}
\small
d\frac{\gamma}{\gamma+\lambda}
<
N_t\rho_{\mathrm{full}}.
\label{eq:factorized_exposure_condition}
\end{equation}
\end{proposition}
\noindent
\textit{Proof.}
See Appendix \ref{app:memory_proof}.

Theorem~\ref{thm:factorized_feedback_density} characterizes feedback
concentration under a bounded active dimension. Proposition~
\ref{thm:factorized_reward_exposure} separates two factors governing
persistent erroneous occupancy: the active dimension $d$ controls the
maximum number of exposed coordinates, while $\gamma$ and $\lambda$
describe the contamination and correction dynamics of a generic
replacement process. For a four-coordinate state, the conditional
occupancy bound is obtained by setting $d=4$.

\subsection{Practical Implementation of RoMeRL}

We implement \textsc{RoMeRL} as an online memory state. For task $g$, let
$
\mathcal{D}_{g,t}
=
\{m_i\}_{i=1}^{n_{g,t}}
$
denote the trajectories observed by interaction $t$, where each $m_i$ has
outcome $y_i\in\{0,1\}$, efficiency $\ell_i$, temporal index $t_i$, and
utility $Q_i$. The factorized state is
\begin{equation}
\small
\mathbf{Z}_{g,t}
=
\begin{bmatrix}
z_{g,t}^{+,\mathrm{C}}
&
z_{g,t}^{+,\mathrm{A}}
\\
z_{g,t}^{-,\mathrm{C}}
&
z_{g,t}^{-,\mathrm{A}}
\end{bmatrix},
\qquad
z_{g,t}^{o,d}
=
\left(
m_{g,t}^{o,d},
Q_{g,t}^{o,d}
\right).
\label{eq:factorized_state_implementation}
\end{equation}
Each non-empty coordinate stores one representative and its utility,
defining the active retrieval support
\begin{equation}
\small
\mathcal{A}_{g,t}
=
\left\{
m_{g,t}^{o,d}
:
(o,d)\in\mathcal{O}\times\mathcal{D},
\;
m_{g,t}^{o,d}\neq\varnothing
\right\},
|\mathcal{A}_{g,t}|
\leq
4.
\label{eq:factorized_active_support}
\end{equation}

Each coordinate starts from $Q_{\mathrm{init}}$. An incoming
representative inherits the current utility as a warm start and resets
$n_i^{\mathrm{post}}$ to zero; subsequent outcomes adapt the inherited
value. For query $x_t$, RoMeRL ranks active memories by weighted
similarity and utility:
\begin{equation}
\small
\begin{split}
\operatorname{score}_t\!\left(m_{g,t}^{o,d}\right)
&=
(1-\omega_Q)
\cos\!\left(
e(x_t),
e\!\left(m_{g,t}^{o,d}\right)
\right)
+
\omega_Q Q_{g,t}^{o,d},
\\
\mathcal{S}_t
&=
\operatorname{TopK}_{k_{\mathrm{ret}}}
\left(
\mathcal{A}_{g,t};
\operatorname{score}_t
\right).
\end{split}
\label{eq:factorized_memory_retrieval}
\end{equation}
Here, $e(\cdot)$ is the embedding encoder and
$\omega_Q\in[0,1]$ weights the learned utility. All non-empty coordinates
are candidates, and the top $k_{\mathrm{ret}}$ memories are retrieved.

\paragraph{Positive Consolidated Coordinate (PCC).}
The coordinate $z_{g,t}^{+,\mathrm{C}}$ preserves a globally consolidated positive reference. Among successful trajectories, it retains the most efficient one:
\begin{equation}
\small
m_{g,t}^{+,\mathrm{C}}
=
\arg\min_{m_i\in\mathcal{D}_{g,t}:y_i=1}
\ell_i.
\label{eq:positive_consolidated}
\end{equation}
When a more efficient successful trajectory is observed, it replaces the
current representative and inherits the current PCC utility as a warm
start. This preserves the accumulated utility state of the consolidated
positive coordinate, while subsequent task-outcome updates adapt it to the
new representative.

\paragraph{Positive Adaptive Coordinate (PAC).}
The coordinate $z_{g,t}^{+,\mathrm{A}}$ records a positive transition from failure to success. Let
$
t_{g}^{\mathrm{fail}}
=
\min_{m_i\in\mathcal{D}_{g,t}:y_i=0}
t_i
\label{eq:first_failure_time}
$
be the first observed failure time. The adaptive positive coordinate retains the earliest successful trajectory following this failure:
\begin{equation}
\small
m_{g,t}^{+,\mathrm{A}}
=
\arg\min_{\substack{
m_i\in\mathcal{D}_{g,t}:y_i=1\\
t_i>t_g^{\mathrm{fail}}
}}
t_i.
\label{eq:positive_adaptive}
\end{equation}
Unlike the PCC, this coordinate is selected by temporal transition rather than global efficiency and therefore captures how the agent first crosses a failure-to-success boundary.

\paragraph{Negative Consolidated Coordinate (NCC).}
The coordinate $z_{g,t}^{-,\mathrm{C}}$ retains failed experience that has accumulated positive downstream utility evidence. A failed trajectory is eligible for consolidation only when its utility exceeds the negative initialization threshold $Q_{\mathrm{init}}^{-}$. Among eligible failures, the state retains the one with the highest utility:
\begin{equation}
\small
m_{g,t}^{-,\mathrm{C}}
=
\arg\max_{\substack{
m_i\in\mathcal{D}_{g,t}:y_i=0\\
Q_i>Q_{\mathrm{init}}^{-}
}}
Q_i.
\label{eq:negative_consolidated}
\end{equation}
Operationally, a trajectory occupying the negative--adaptive coordinate is promoted to this coordinate when its utility exceeds both $Q_{\mathrm{init}}^{-}$ and the utility of the current consolidated negative representative.

\paragraph{Negative Adaptive Coordinate (NAC).}
The coordinate $z_{g,t}^{-,\mathrm{A}}$ tracks the agent's current failure state by retaining the most recent failed trajectory:
\begin{equation}
\small
m_{g,t}^{-,\mathrm{A}}
=
\arg\max_{m_i\in\mathcal{D}_{g,t}:y_i=0}
t_i.
\label{eq:negative_adaptive}
\end{equation}
Each newly observed failure replaces the previous coordinate content. This temporal update makes the coordinate responsive to recent errors without allowing every failed trajectory to become a persistent utility variable.
\begin{table*}[htb]
\centering
\scriptsize
\setlength{\tabcolsep}{2.5pt}
\setlength{\abovecaptionskip}{6pt}
\caption{
\textbf{Main results over 10 epochs.}
We compare \textsc{RoMeRL} with non-learning and learning-based
agent-memory baselines. LAB reports \textbf{Last-Epoch SR / CSR}, while
ALFWorld reports SR for six task types: P\&P, examine, clean, heat, cool,
and Pick-2. AppWorld reports Task Goal Completion (TGC) and Scenario
Goal Completion (SGC). \textbf{Overall} is the macro-average over ten
evaluation units: last-epoch SR for the two LAB tasks, SR for the six
ALFWorld task types, and TGC and SGC for AppWorld.
}
\label{tab:runtime_main_results}
\resizebox{\textwidth}{!}{
\begin{tabular}{l cc cccccc cc c}
\toprule

\multirow{2}{*}{\textbf{Method}}
&
\multicolumn{2}{c}{\textbf{Lifelong Agent Bench}}
&
\multicolumn{6}{c}{\textbf{ALFWorld}}
&
\multicolumn{2}{c}{\textbf{AppWorld}}
&
\multirow{2}{*}{\textbf{\shortstack[c]{Overall\\Avg.}}}
\\

\cmidrule(lr){2-3}
\cmidrule(lr){4-9}
\cmidrule(lr){10-11}

&
\begin{tabular}{@{}c@{}}
\textbf{OS}\\
\textbf{Last / CSR}
\end{tabular}
&
\begin{tabular}{@{}c@{}}
\textbf{DB}\\
\textbf{Last / CSR}
\end{tabular}
&
\textbf{P\&P}
&
\textbf{Examine}
&
\textbf{Clean}
&
\textbf{Heat}
&
\textbf{Cool}
&
\textbf{Pick-2}
&
\textbf{TGC}
&
\textbf{SGC}
&
\\

\cmidrule(lr){1-12}

\textit{Model}
&
\texttt{DS-V4-flash}
&
\texttt{DS-V4-flash}
&
\multicolumn{6}{c}{\texttt{GPT-5.4-mini}}
&
\multicolumn{2}{c}{\texttt{GPT-5.6-luna}}
&
--
\\

\midrule

No Memory
&
0.646
&
0.550
&
0.883
&
0.827
&
0.861
&
0.850
&
0.855
&
0.788
&
0.238
&
0.224
&
0.672
\\

Pass@10
&
-- / 0.756
&
-- / 0.906
&
--
&
--
&
--
&
--
&
--
&
--
&
--
&
--
&
--
\\

RAG
&
0.700 / 0.752
&
0.556 / 0.844
&
0.891
&
0.834
&
0.868
&
0.855
&
0.858
&
0.796
&
0.170
&
0.143
&
0.667
\\

Mem0
&
0.691 / 0.733
&
0.575 / 0.841
&
0.897
&
0.841
&
0.873
&
0.858
&
0.872
&
0.805
&
0.215
&
0.224
&
0.685
\\

\midrule

MemP
&
0.768 / 0.796
&
0.631 / 0.942
&
0.899
&
0.847
&
0.871
&
0.862
&
0.875
&
0.810
&
0.245
&
0.204
&
0.701
\\

MemRL
&
0.808 / 0.820
&
0.632 / 0.934
&
0.908
&
0.855
&
0.887
&
\textbf{0.865}
&
0.871
&
0.812
&
\textbf{0.313}
&
0.286
&
0.724
\\

\midrule

\textit{RoMeRL (ours)}
&
\textbf{0.824} / \textbf{0.838}
&
\textbf{0.680} / \textbf{0.952}
&
\textbf{0.968}
&
\textbf{0.957}
&
\textbf{0.901}
&
0.862
&
\textbf{0.880}
&
\textbf{0.826}
&
0.306
&
\textbf{0.326}
&
\textbf{0.753}
\\

\bottomrule
\end{tabular}
}
\end{table*}
\paragraph{Runtime Utility Update.}
Given the retrieved set $\mathcal{S}_t\subseteq\mathcal{A}_{g,t}$ and the task-level outcome reward $r_t$, each retrieved coordinate is updated by
\begin{equation}
\small
Q_{g,t+1}^{o,d}
=
Q_{g,t}^{o,d}
+
\alpha\,
\mathbbm{1}
\left[
m_{g,t}^{o,d}\in\mathcal{S}_t
\right]
\left(
r_t-Q_{g,t}^{o,d}
\right).
\label{eq:factorized_utility_update}
\end{equation}
The updated utility is combined with semantic similarity for subsequent
retrieval. Suppose that a new representative $m_j$ inherits the initial
value $Q_{j,0}$ and has a stationary raw-return target $\mu_j$. After $s$
post-replacement utility updates,
\begin{equation}
\small
\mathbb{E}
\left[
Q_{j,s}-\mu_j
\right]
=
(1-\alpha)^s
\left(
Q_{j,0}-\mu_j
\right).
\label{eq:inherited_utility_adaptation}
\end{equation}
Thus, the inherited value provides a warm start, while subsequent
task-outcome feedback progressively adapts the utility to the new
representative. RoMeRL retains the standard Q update over a
fixed set of semantic coordinates.

\begin{figure*}[htb]
    \centering
    \begin{subfigure}{0.47\textwidth}
        \centering
        \includegraphics[width=\linewidth]{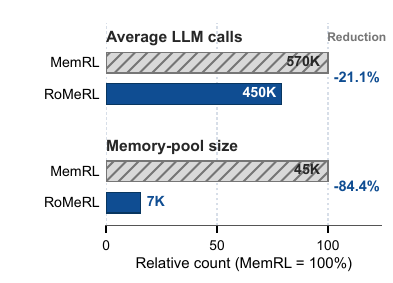}
        \caption{
            Efficiency comparison between MemRL and \textsc{RoMeRL}.
            We report the average number of LLM calls and the memory-pool size.
            The horizontal axis uses a logarithmic scale, with lower values
            indicating better efficiency.
        }
        \label{fig:efficiency_comparison}
    \end{subfigure}
    \hfill  
    \begin{subfigure}{0.47\textwidth}
        \centering
        \includegraphics[width=\linewidth]{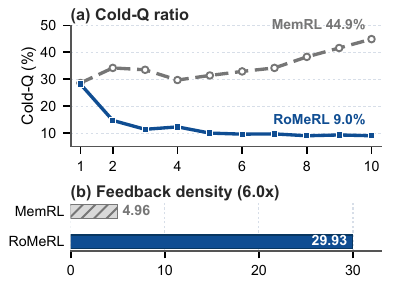}
        \caption{
            Feedback density and Cold-Q ratio on the OS task.
            The Cold-Q ratio is the fraction of current representatives with no
            direct utility update since admission, i.e.,
            $n_i^{\mathrm{post}}=0$, while feedback density is their average number
            of post-admission updates.
        }
        \label{fig:coldq_feedback_density}
    \end{subfigure}
    \caption{Performance and efficiency analysis.} 
    \label{fig:combined} 
\end{figure*}
\section{Experiments}
We compare RoMeRL with retrieval-based memory
(RAG~\citep{wu2020rag}, Mem0~\cite{chhikara2025mem0}), agentic memory (MemP~\citep{fang2026memp}),
test-time scaling (Pass@$k$), and runtime memory reinforcement learning
(MemRL~\citep{zhang2026memrl}), using frozen LLM backbones throughout.
The evaluation covers the OS and DB tasks of LifelongAgentBench, the six
task types of ALFWorld, and AppWorld. All methods within each benchmark use
the same backbone and interaction budget. For LifelongAgentBench, we report
last-epoch Success Rate (SR) and Cumulative Success Rate (CSR), defined as
the proportion of tasks solved at least once across epochs. For ALFWorld,
we report task-type SR, while AppWorld is evaluated using Task Goal
Completion (TGC) and Scenario Goal Completion (SGC). Appendix~B provides
additional details.
\subsection{Main Experiment}

Table~\ref{tab:runtime_main_results} shows that RoMeRL achieves the highest
overall average score of $0.753$, outperforming MemRL, the strongest
baseline evaluated on all three benchmarks, by $2.9$ percentage points
($0.724$). RoMeRL obtains the best last-epoch SR on both LifelongAgentBench
tasks and on five of the six ALFWorld task types. On AppWorld, it improves
SGC from $0.286$ to $0.326$, a gain of $4.0$ percentage points, while
remaining competitive in TGC ($0.306$ versus $0.313$). These results indicate
that the fixed-dimensional semantic-coordinate state transfers across
terminal, embodied, and multi-application environments while keeping the
LLM backbone frozen.

\paragraph{Feedback Utilization, MRT Robustness, and Efficiency.}
Figure~\ref{fig:coldq_feedback_density} reveals markedly different
feedback dynamics. MemRL's Cold-Q ratio increases from approximately
$29\%$ to $44.9\%$, whereas RoMeRL reduces this ratio from approximately
$28\%$ to $9.0\%$ while increasing feedback density from $4.96$ to
$29.93$ ($6.0\times$). As shown in
Figure~\ref{fig:efficiency_comparison}, this improved feedback
utilization is accompanied by substantially lower resource costs.
Compared with MemRL, RoMeRL reduces the average number of LLM calls by
$120$K, from $570$K to $450$K ($21.1\%$), and reduces the memory-pool
size by $38$K, from $45$K to $7$K ($84.4\%$). The MRT stress test in
Table~\ref{tab:mrt_noise_stress_test} further demonstrates the
robustness of this reduced-order representation. Adding UCB to MemRL
increases the number of positive noise updates from $3.7$ to $7.2$ and
the final noise ratio from $1.02\%$ to $1.20\%$. In contrast, RoMeRL
limits these quantities to $2.4$ and $0.15\%$, respectively, while
achieving the highest success rate of $82.0\%$. Additional details are
provided in Appendix~B.

\paragraph{Cross-Model Memory Transfer.}
Table~\ref{tab:cross_model_memory_transfer} shows the same pattern in
all four model--task combinations: transferring the frozen memory state
improves the score and reduces the average number of execution steps.
The factorized state therefore carries procedural information that is
useful across LLM backbones, both for solving more tasks and for reaching
solutions more directly. Because each interaction step typically
requires another LLM invocation, the lower step counts also reduce
inference cost.
\begin{table}[htb]
\centering
\small
\renewcommand{\arraystretch}{1.08}
\caption{Controlled MRT stress test on the OS task. The first-round memory
pool contains $10\%$ noisy entries, and results are reported after ten
training rounds. Positive Noise Updates denotes the average number of
positive utility updates received by noisy entries over the ten rounds,
whereas Final Noise Ratio denotes the percentage of noisy entries in the
memory pool at the end of round 10.}
\label{tab:mrt_noise_stress_test}
\begin{tabular}{lccc}
\toprule
\textbf{Method}
& \shortstack{\textbf{Round-10}\\\textbf{SR (\%)} $\uparrow$}
& \shortstack{\textbf{Positive Noise}\\\textbf{Updates} $\downarrow$}
& \shortstack{\textbf{Final Noise}\\\textbf{Ratio (\%)} $\downarrow$} \\
\midrule
MemRL       & 79.2 & 3.7 & 1.02 \\
MemRL + UCB & 78.4 & 7.2 & 1.20 \\
RoMeRL & \textbf{82.0} & \textbf{2.4} & \textbf{0.15} \\
\bottomrule
\end{tabular}
\end{table}
\begin{table}[htb]
\centering
\small
\renewcommand{\arraystretch}{1.08}
\caption{
Cross-model memory transfer on LifelongAgentBench OS and DB tasks. Results
with and without frozen transferred memory are reported as
\textbf{Validation Score / Average Steps} (higher / lower is better), and $\Delta$
denotes the absolute score gain over the base agent.
}
\label{tab:cross_model_memory_transfer}
\begin{tabular}{lccc}
\toprule
\textbf{Inference Model}
& \textbf{Base}
& \textbf{Transfer}
& \textbf{Gain ($\Delta$)}
\\
\midrule

\multicolumn{1}{c}{\textit{\textsc{LifelongAgentBench}--OS}}
\\
\addlinespace[1pt]

\texttt{GPT-5.4-mini}
& 67.0 / 3.23
& 81.6 / 2.22
& \textbf{+14.6 / -1.01}
\\

\texttt{Gemini-3.5-flash}
& 74.0 / 4.53
& 81.4 / 3.02
& \textbf{+7.4 / -1.51}
\\

\midrule

\multicolumn{1}{c}{\textit{\textsc{LifelongAgentBench}--DB}}
\\
\addlinespace[1pt]

\texttt{GPT-5.4-mini}
& 93.0 / 2.15
& 96.8 / 2.00
& \textbf{+3.8 / -0.15}
\\

\texttt{Gemini-3.5-flash}
& 96.2 / 2.44
& 97.6 / 2.18
& \textbf{+1.4 / -0.26}
\\

\bottomrule
\end{tabular}
\end{table}
\subsection{Ablation Study}

To assess the factorized state, we ablate NCC,
$z_{g,t}^{-,\mathrm C}$, and PAC, $z_{g,t}^{+,\mathrm A}$, on the OS task
while retaining PCC and NAC as the basic positive and negative anchors.
This isolates the complementary roles of consolidated negative evidence
and adaptive positive transitions. The results are shown in Figure 5.
\begin{wrapfigure}{r}
{0.47\columnwidth}   
    \centering
    \includegraphics[width=\linewidth]{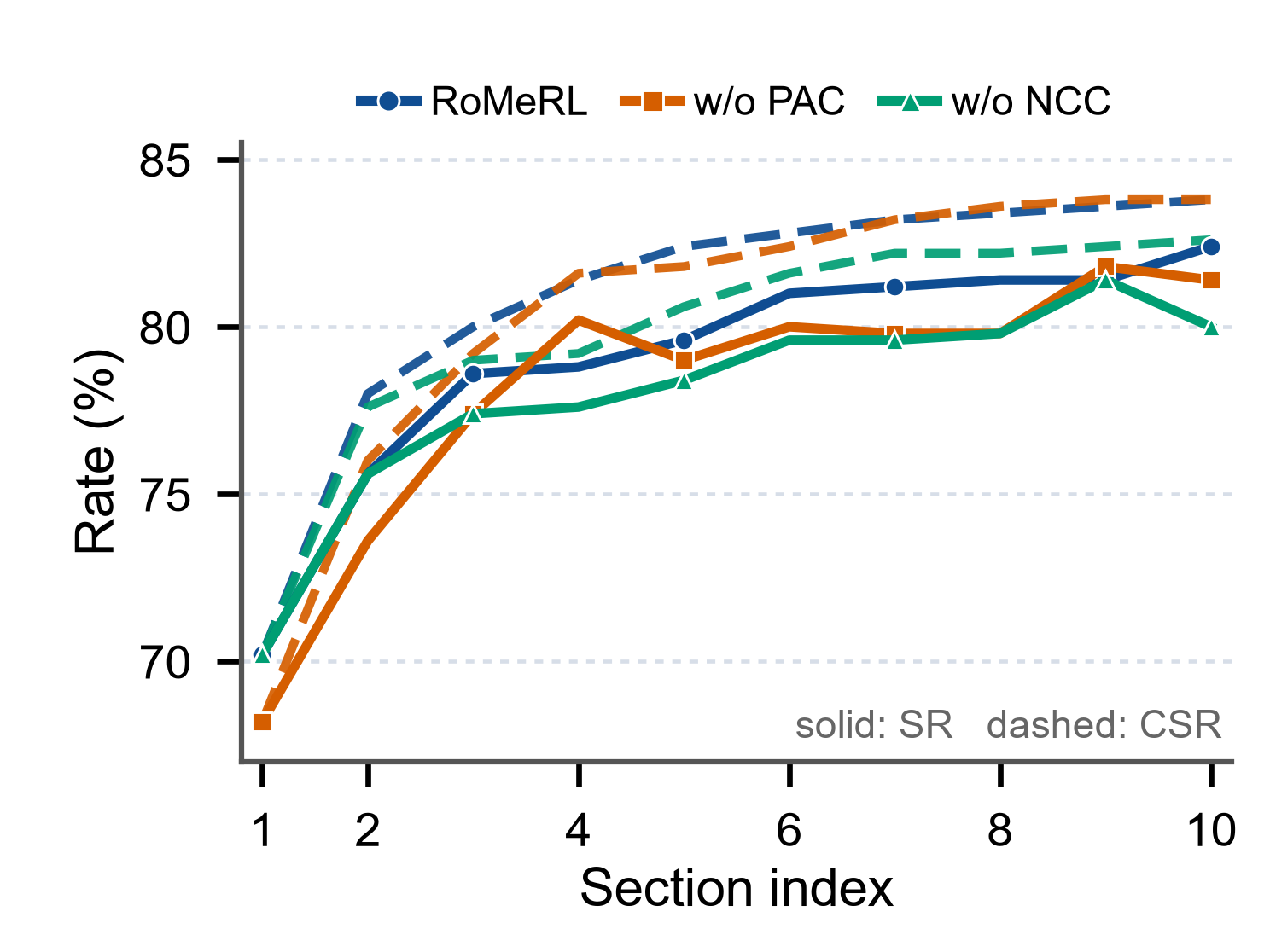}
    \caption{Coordinate ablation on the OS task. Solid and dashed
             curves denote SR and CSR.}
    \label{fig:ablation1}
    \vspace{-3em}
\end{wrapfigure}
\paragraph{Negative Consolidated Coordinate.}
Removing NCC lowers both Last-Epoch Success Rate and CSR
(Figure~\ref{fig:ablation1}). NCC retains failure-derived experiences with high downstream raw-return
utility, providing a compact source of potentially reusable negative
information, allowing the agent to reuse informative
negative evidence without storing every failed trajectory. In the full
factorized state, NCC accounts for $27.14\%$ of occupied coordinates on
OS and $42.78\%$ on DB. Its ablation and high
occupancy together indicate that consolidated negative evidence supports
both current performance and cumulative task coverage.
\paragraph{Positive Adaptive Coordinate.}
Removing PAC mainly reduces Last-Epoch Success Rate, with little change
in CSR. PAC preserves the first successful trajectory observed after a
failure, thereby retaining the recovery pattern. It occupies only $8.21\%$ of the coordinates on
OS and $5.05\%$ on DB, but its removal still
lowers current performance. PAC thus provides a sparse transition signal
that helps the agent reproduce previously discovered solutions. The NCC
and PAC results show complementary roles for consolidated negative
evidence and adaptive positive evidence.
\begin{figure}[htb]
	\centering
	\includegraphics[width=0.85\linewidth]{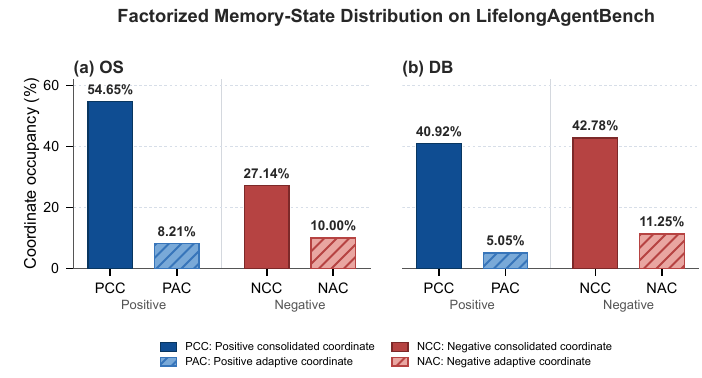}
	\caption{Occupancy distribution of the factorized memory state on the OS
and DB tasks.
The bars report the percentage of occupied active-memory coordinates
assigned to the PCC, PAC, NCC, and NAC.
}
	\label{fig:zhanbi}
\end{figure}
\section{Conclusion and Limitations}
We introduced RoMeRL, a reduced-order memory reinforcement learning
framework for self-evolving LLM agents. Our analysis identifies the MRT,
where outcome-driven Q updates reinforce co-retrieved memories with little
causal contribution. Exploration reduces estimation variance but not
attribution bias, while full-pool learning requires growing feedback.
RoMeRL instead uses compact per-task coordinates factorized by outcome
polarity and memory dynamics. Across LifelongAgentBench, ALFWorld, and
AppWorld, it achieves the highest overall average score while reducing MRT
exposure, memory size, and inference cost.

RoMeRL still relies on outcome-level rewards and therefore does not fully
resolve causal credit assignment. Moreover, estimating the
transition quantities $\gamma$ and $\lambda$ in
Proposition~\ref{thm:factorized_reward_exposure} requires coordinate-level
causal labels from paired counterfactual rollouts or equivalent attribution.
Future work may combine such attribution with finer-grained rewards and
evaluate open-ended, longer-horizon settings.
\newpage
\bibliographystyle{iclr2027_conference}
\bibliography{romerl}
\clearpage

\appendix

\setcounter{secnumdepth}{2}

\begin{center}
{\LARGE\bfseries Appendix}
\end{center}
\vspace{0.5em}

\section{Proofs for RoMeRL}
\label{app:memory_proof}

In this appendix, we provide detailed proofs for the theoretical results in
Sections 3 and 4. We first prove the bias-variance
decomposition of memory utility estimation, and then discuss its implication
for outcome-driven memory updates.

\subsection{Proof of Theorem~\ref{thm:bias_variance_memory}}
\label{app:proof_th1}

\begin{proof}
Fix a memory $m_i$ and a positive visit count $n_i$. Write the rewards observed
on its retrieval events as $R_{i,1},\ldots,R_{i,n_i}$. Under the assumptions of
Theorem~\ref{thm:bias_variance_memory}, these variables are independent, with
common mean $\mu_i$ and variance $\sigma_i^2$, and
\[
    \hat Q_i=\frac{1}{n_i}\sum_{j=1}^{n_i}R_{i,j}.
\]
Linearity of expectation and independence give
\begin{align*}
    \mathbb E[\hat Q_i]
    &= \frac{1}{n_i}\sum_{j=1}^{n_i}\mathbb E[R_{i,j}]
     = \mu_i, \\
    \operatorname{Var}(\hat Q_i)
    &= \frac{1}{n_i^2}\sum_{j=1}^{n_i}\operatorname{Var}(R_{i,j})
     = \frac{\sigma_i^2}{n_i}.
\end{align*}
Thus, $\hat Q_i$ is centered at the observational utility rather than the
interventional marginal utility. By definition, $a_i=\mu_i-v_i^1$ and
$\mu_i-\theta_i=v_i^0+a_i$. Adding and subtracting $\mu_i$ gives
\begin{align*}
\mathbb E[(\hat Q_i-v_i^1)^2]
&=
\mathbb E[(\hat Q_i-\mu_i)^2]
+
(\mu_i-v_i^1)^2
\\
&=
\frac{\sigma_i^2}{n_i}+a_i^2,
\\[2mm]
\mathbb E[(\hat Q_i-\theta_i)^2]
&=
\mathbb E[(\hat Q_i-\mu_i)^2]
+
(\mu_i-\theta_i)^2
\\
&=
\frac{\sigma_i^2}{n_i}
+
(v_i^0+a_i)^2.
\end{align*}
The cross terms vanish because
$\mathbb E[\hat Q_i-\mu_i]=0$.
\end{proof}

\paragraph{Random visit counts.}
The theorem treats $n_i$ as fixed. If the visit count is random, the same
identity holds conditionally provided that, given $n_i$, the selected rewards
retain the stated mean, variance, and independence properties. Under this
non-informative sampling condition and on the event $n_i>0$,
\[
\mathbb E[(\hat Q_i-\theta_i)^2\mid n_i]
=(v_i^0+a_i)^2+\frac{\sigma_i^2}{n_i},
\]
and averaging over $n_i$ replaces the variance term by
$\mathbb E[\sigma_i^2/n_i]$. This conditional step is not automatic under an
adaptive stopping rule that depends on previously observed rewards; such a
rule would require a separate martingale or optional-stopping analysis.

\paragraph{Why additional feedback does not remove the credit gap.}
For fixed $v_i^0$ and $a_i$, the variance term vanishes as $n_i$ grows,
whereas the total credit gap remains:
\[
\lim_{n_i\to\infty}
\mathbb E[(\hat Q_i-\theta_i)^2]
=
(v_i^0+a_i)^2.
\]
Moreover, $\hat Q_i\xrightarrow{p}\mu_i$. More feedback therefore makes
the empirical estimate increasingly precise around the raw-return target
$\mu_i$; it does not remove the task-level baseline $v_i^0$, the
observational attribution bias $a_i$, or their resulting gap to the
marginal contribution $\theta_i$.

\paragraph{Positive reinforcement under the memory-reward trap.}
Consider a trapped memory, for which
$\theta_i\leq0$ but $\mu_i-v_i^0=\theta_i+a_i>0$. Let
$\mathcal F_t$ be the pre-reward history and define
$\mu_{i,t}=\mathbb E[R_t\mid\mathcal F_t,m_i\in\mathcal S_t]$. Since
$Q_{i,t}$ is $\mathcal F_t$-measurable, the exponential update in
Equation~\eqref{eq:ema_memory_update} satisfies
\begin{equation}
    \mathbb E[Q_{i,t+1}-Q_{i,t}\mid
    \mathcal F_t,m_i\in\mathcal S_t]
    =\alpha(\mu_{i,t}-Q_{i,t}).
    \label{eq:appendix_ema_expected_drift}
\end{equation}
Under the stationary reference model $\mu_{i,t}=\mu_i$, the conditional
drift is positive whenever $Q_{i,t}<\mu_i$. This raw-return drift includes
the task baseline $v_i^0$. The MRT is not identified by the positivity of
$\mu_i$ alone, but by the sign reversal
$\theta_i\leq0<\mu_i-v_i^0$. If the retrieval score is nondecreasing in
$Q_i$, repeated raw-return updates can make such a memory more likely to
be retrieved again.

\paragraph{Fixed-step EMA variance.}
Equation~\eqref{eq:appendix_ema_expected_drift} is a statement about expected
drift, not a vanishing-variance guarantee. To see the distinction, suppose the
rewards at the $n$ update events are i.i.d. with mean $\mu_i$ and variance
$\sigma_i^2$, and let $Q_{i,0}$ be deterministic. Unrolling the recursion gives
\[
    Q_{i,n}=(1-\alpha)^nQ_{i,0}
    +\alpha\sum_{r=1}^{n}(1-\alpha)^{n-r}R_{i,r}.
\]
Consequently,
\begin{align*}
    \mathbb E[Q_{i,n}]
    &= (1-\alpha)^nQ_{i,0}
       +\bigl(1-(1-\alpha)^n\bigr)\mu_i
       \longrightarrow \mu_i, \\
    \operatorname{Var}(Q_{i,n})
    &= \alpha^2\sigma_i^2
       \sum_{r=0}^{n-1}(1-\alpha)^{2r} \\
    &= \frac{\alpha\sigma_i^2}{2-\alpha}
       \left(1-(1-\alpha)^{2n}\right)
       \longrightarrow \frac{\alpha\sigma_i^2}{2-\alpha}.
\end{align*}
Thus, a fixed-step runtime EMA approaches $\mu_i$ in mean but generally retains
a nonzero variance floor. The empirical-mean results in
Theorems~\ref{thm:bias_variance_memory} and
\ref{thm:factorized_feedback_density} concern a different estimator; the MRT
argument for the runtime update relies only on its expected drift toward the
observational target.

\subsection{Proof of Theorem~\ref{thm:full_pool_sample_complexity}}
\label{app:proof_full_pool_sample_complexity}

\begin{proof}
Write $N=N_t$ for the current full-pool dimension. For a fixed memory
$m_i$, Hoeffding's inequality gives
\begin{equation}
    \Pr\!\left(|\hat Q_i-\mu_i|
>\epsilon\right)
    \leq 2\exp(-2n_i\epsilon^2).
    \label{eq:hoeffding_single_memory}
\end{equation}
The estimates need not be independent across memories: applying the union
bound directly yields
\[
    \Pr\!\left(\max_{i\in[N_t]}|\hat Q_i-\mu_i|
>\epsilon\right)
    \leq \sum_{i=1}^{N_t}2\exp(-2n_i\epsilon^2).
\]
It is therefore enough to choose a common per-memory count
\begin{equation}
    n^\star
    =
    \left\lceil
    \frac{1}{2\epsilon^2}\log\frac{2N_t}{\delta}
    \right\rceil
    \quad\text{and ensure}\quad
    n_i\geq n^\star\ \text{for every }i.
    \label{eq:sample_per_memory}
\end{equation}
Substitution into the union bound gives a failure probability no larger than
$\delta$, which proves the simultaneous guarantee in
Equation~\eqref{eq:uniform_full_pool_estimation}.

It remains to translate the per-memory requirement into budgets. A balanced
allocation that assigns $n^\star$ feedback signals to each of the $N_t$
memories uses
\[
    F_T^\star=N_tn^\star
    \leq
    N_t\left(
    1+\frac{1}{2\epsilon^2}\log\frac{2N_t}{\delta}
    \right).
\]
Hence, in the nontrivial regime $0<\epsilon\leq1$ and $0<\delta<1$,
\[
    F_T^\star
    =
    O\!\left(
    \frac{N_t}{\epsilon^2}\log\frac{N_t}{\delta}
    \right).
\]
If at most $k$ memory-level updates are packed into one trajectory, the same
balanced allocation can be scheduled in
$T^\star=\lceil F_T^\star/k\rceil$ trajectories. Thus
\[
    T^\star
    =
    O\!\left(
    \frac{N_t}{k\epsilon^2}\log\frac{N_t}{\delta}
    \right),
\]
up to the immaterial final partially filled trajectory. These are sufficient
budgets, matching the statement of
Theorem~\ref{thm:full_pool_sample_complexity}; no minimax lower bound is
claimed.
\end{proof}

\subsection{Proof of Theorem~\ref{thm:factorized_feedback_density}}
\label{app:proof_factorized_feedback_density}

\begin{proof}
Let $D$ denote the number of utility coordinates being maintained, and let
$n_j$ be the number of feedback signals assigned to coordinate $j$. The
realized number of coordinate-level updates is
\[
    F_T=\sum_{j=1}^{D}n_j\leq kT.
\]
Irrespective of how these updates are distributed, the average count per
coordinate is exactly
\[
    \bar n_D=\frac{1}{D}\sum_{j=1}^{D}n_j=\frac{F_T}{D}.
\]
When the available update budget is used to constant order,
$F_T\asymp kT$, this becomes $\bar n_D\asymp kT/D$, which is the
approximation used in the theorem.

Now compare a trajectory-indexed state of dimension $N_t$ with a factorized
state of dimension four under the same realized feedback budget $F_T$. Their
average feedback counts satisfy
\begin{equation}
    \bar n_{\mathrm{fact}}
    =\frac{F_T}{4}
    =\frac{N_t}{4}\frac{F_T}{N_t}
    =\frac{N_t}{4}\bar n_{\mathrm{full}}.
    \label{eq:appendix_feedback_ratio}
\end{equation}
This identity concerns feedback density and does not require identical
coordinate-level visit counts.

The variance comparison does require the approximately balanced allocation
stated in the theorem. More precisely, suppose there is a constant
$c\in(0,1]$, independent of $D$ and $T$, such that
$n_j\geq cF_T/D$ for every maintained coordinate. If the conditional reward
variance of each coordinate is at most $\sigma^2$, then the empirical-mean
estimator from Theorem~\ref{thm:bias_variance_memory} obeys
\begin{align*}
    \frac{1}{D}\sum_{j=1}^{D}\operatorname{Var}(\hat Q_j)
    &\leq \frac{1}{D}\sum_{j=1}^{D}\frac{\sigma^2}{n_j} \\
    &\leq \frac{\sigma^2D}{cF_T}
     = O\!\left(\frac{\sigma^2D}{kT}\right),
\end{align*}
where the final form again uses $F_T\asymp kT$. Setting $D=N_t$ gives
$O(\sigma^2N_t/(kT))$ for the trajectory-indexed state, whereas setting
$D=4$ gives $O(4\sigma^2/(kT))$ for the factorized state.

If the update budget is not saturated, the same statements remain valid with
$F_T$ in place of $kT$. If feedback is highly unbalanced, the density identity
in Equation~\eqref{eq:appendix_feedback_ratio} still holds, but average
feedback alone no longer implies the stated variance bound. This is why the
balanced-allocation condition is explicit in the theorem.
\end{proof}

\subsection{Proof of Proposition~\ref{thm:factorized_reward_exposure}}
\label{app:proof_factorized_reward_exposure}

\begin{proof}
For the trajectory-indexed pool, define
$X_i=\mathbbm 1[\theta_i\leq0,\ \mu_i-v_i^0>0]$. By the definition of
$\rho_{\mathrm{full}}$, $\mathbb E[X_i]=\rho_{\mathrm{full}}$. Therefore,
without requiring independence among memories,
\begin{equation}
    \mathbb E\!\left[\sum_{i=1}^{N_t}X_i\right]
    =\sum_{i=1}^{N_t}\mathbb E[X_i]
    =N_t\rho_{\mathrm{full}}.
    \label{eq:appendix_full_pool_exposure}
\end{equation}

Now consider one coordinate of the active state. Let $p_{j,t}$ be the
probability that coordinate $j$ is erroneous after its $t$-th transition.
An erroneous coordinate remains erroneous with probability at most
$1-\lambda$, while a clean coordinate becomes erroneous with probability at
most $\gamma$. Hence
\begin{equation}
    p_{j,t+1}
    \leq (1-\lambda)p_{j,t}
       +\gamma(1-p_{j,t}).
    \label{eq:appendix_coordinate_transition}
\end{equation}
Under the stationary transition model in the theorem, write the stationary
erroneous probability as $p_j^\star$. Applying
Equation~\eqref{eq:appendix_coordinate_transition} at stationarity and
rearranging gives
\[
    (\gamma+\lambda)p_j^\star\leq\gamma,
    \qquad
    p_j^\star\leq\frac{\gamma}{\gamma+\lambda}.
\]

Let $Y_j$ indicate that active coordinate $j$ is erroneous in stationarity.
Linearity of expectation again avoids any independence requirement:
\begin{equation}
    \mathbb E\!\left[\sum_{j=1}^{d}Y_j\right]
    =\sum_{j=1}^{d}p_j^\star
    \leq d\frac{\gamma}{\gamma+\lambda}.
    \label{eq:appendix_active_exposure}
\end{equation}
Comparing Equations~\eqref{eq:appendix_full_pool_exposure} and
\eqref{eq:appendix_active_exposure} proves that the active state has strictly
smaller steady-state erroneous occupancy whenever
\[
    d\frac{\gamma}{\gamma+\lambda}
    <N_t\rho_{\mathrm{full}}.
\]
For a four-coordinate active state, setting $d=4$ gives the corresponding
conditional occupancy bound.

The replacement probability also controls persistence at the coordinate
level. Conditional on a coordinate being erroneous, the probability that it
remains erroneous for at least $r$ further transitions is at most
$(1-\lambda)^r$. Its expected erroneous residence time is therefore at most
$1/\lambda$. This residence-time bound is consistent with the stationary occupancy
result in Proposition~\ref{thm:factorized_reward_exposure}.
\end{proof}

\section{Implementation Details}
We facilitate reproducibility by documenting the exact model versions, hyperparameter settings, and environmental configurations used in our experiments.

\subsection{Model Specifications}
We performed all LLM reasoning and generation tasks using the models in Table~\ref{tab:model_specs}. To maximize reproducibility, we accessed them through the official APIs with a fixed temperature, ensuring deterministic outputs where feasible.

\begin{table}[htb]
    \centering
    \caption{Model and API Configurations.}
    \label{tab:model_specs}
    \begin{tabular}{l l l}
        \toprule
        \textbf{Component} & \textbf{Configuration / Version} & \textbf{Notes} \\
        \midrule
        \textbf{Backbone LLM} & \texttt{DS-V4-flash} & Used for LifelongAgentBench \\
                              & \texttt{GPT-5.4-mini} & Used for ALFWorld \\
                              & \texttt{GPT-5.6-luna} & Used for AppWorld \\
                              
        \midrule
        \textbf{Embedding Model} & \texttt{Text-Embedding-3-Large} & Used for Intent and Query encoding \\
        \midrule
        \textbf{Generation Params} & Temperature $= 0.0$ & General (Greedy decoding) \\
                                   
                                   & Top-p $= 1.0$ & Default \\
        \bottomrule
    \end{tabular}
\end{table}

\subsection{Hyperparameter Settings}

Table~\ref{tab:hyperparams} details the hyperparameters used for RoMeRL and
the baselines on LifelongAgentBench and ALFWorld. The similarity threshold
$\delta$ is adaptive to dataset density: for each reported setting, we
compute the pairwise cosine-similarity distribution of task descriptions and
select the top-$20\%$ quantile as the threshold. This retains only the most
relevant historical experiences during retrieval.

\begin{table}[t]
    \centering
    \caption{Hyperparameter settings for LifelongAgentBench and ALFWorld.}
    \label{tab:hyperparams}
    \resizebox{0.85\linewidth}{!}{%
    \begin{tabular}{l l c c c}
        \toprule
        & & \multicolumn{3}{c}{\textbf{Benchmark Setting}} \\
        \cmidrule(lr){3-5}
        \textbf{Parameter} & \textbf{Description} & \textbf{Lifelong Bench (OS)} & \textbf{Lifelong Bench (DB)} & \textbf{ALFWorld} \\
        \midrule
        \multicolumn{5}{l}{\textit{\textsc{RoMeRL} (Ours)}} \\
        $\alpha$ & Learning Rate & 0.3 & 0.3 & 0.3 \\
        $\omega_Q$ & Q-Weight Balance & 0.5 & 0.5 & 0.5 \\
        $\delta$ & Similarity Threshold & 0.50 & 0.37 & 0.62 \\
        $k_1$ & Cosine Similarity Recall Size & 10 & 10 & 5 \\
        $k_2$ & Final Memory Selection Size & 5 & 5 & 3 \\
        $Q_{init}$ & Initial Q-value & 0.5 & 0.5 & 0.0 \\
        \midrule
        \multicolumn{5}{l}{\textit{Baselines}} \\
        $k_{RAG}$ & Retrieval Top-k & 5 & 5 & 3 \\
        $k_{SelfRAG}$ & Retrieval Top-k & 5 & 5 & 3 \\
        $k_{MemP}$ & Retrieval Top-k & 5 & 5 & 3 \\
        \bottomrule
    \end{tabular}%
    }
\end{table}
\subsection{Data Partitioning}

Table~\ref{tab:data_split} summarizes the dataset partitions used across the
three benchmarks. For LifelongAgentBench, we consider both runtime-learning
and transfer-learning settings and use a fixed 7:3 random split with seed 42
for the OS and DB tasks. ALFWorld and AppWorld follow their official
benchmark splits and are evaluated only in the runtime-learning setting.

\begin{table}[htb]
    \centering
    \footnotesize
    \caption{
    Data splits across benchmarks. LifelongAgentBench uses a fixed random
    split with seed 42. ALFWorld and AppWorld follow their official benchmark
    splits; transfer learning is not evaluated for these two benchmarks.
    }
    \label{tab:data_split}
    \setlength{\tabcolsep}{5pt}
    \renewcommand{\arraystretch}{1.10}
    \resizebox{\linewidth}{!}{%
    \begin{tabular}{l c c l}
        \toprule
        \textbf{Benchmark}
        & \textbf{Runtime Learning}
        & \textbf{Transfer Learning}
        & \textbf{Split / Note} \\
        \midrule

        LifelongAgentBench (OS)
        & 500 tasks
        & 500 tasks
        & 7:3 random split (seed 42) \\

        LifelongAgentBench (DB)
        & 500 tasks
        & 500 tasks
        & 7:3 random split (seed 42) \\

        ALFWorld
        & \shortstack{3,553 train}
        & --
        & Official split (seed 42)\\

        AppWorld
        & \shortstack{105 train / 60 dev}
        & --
        & \shortstack{Official split (seed 42)} \\

        \bottomrule
    \end{tabular}%
    }
\end{table}

\subsection{Benchmark Details}
\label{app:benchmarks}
We evaluate performance across three benchmarks spanning operating-system
and database interaction, embodied household decision-making, and
compositional workflows over interconnected applications.

\paragraph{LifelongAgentBench (LAB)~\citep{zheng2025lifelongagentbench}:} LifelongAgentBench is designed to evaluate lifelong
learning and experience reuse in interactive terminal-based environments. It contains 1,396 total task instances
across three environments: Database (DB), Operating System (OS), and Knowledge Graph (KG). Following
prior work~\citep{zhang2026memrl}, we focus on the DB and OS subsets.
The DB subset (500 tasks) evaluates 22 SQL-related skills: basic SELECT, filtering (WHERE), grouping
(GROUP BY), sorting (ORDER BY), aggregation (COUNT / SUM / AVG / MAX / MIN), nested subqueries,
multi-table JOINs, set operations (UNION / INTERSECT), and data manipulation (INSERT / UPDATE / DELETE).
Execution results are verified automatically via SQL engine output.
The OS subset (500 tasks) evaluates 29 Bash-command skills: file and directory operations (ls, cp, mv, find),
permission management (chmod, chown), user and group management (useradd, groupmod), text processing
(grep, sed, awk, wc), compression (tar, gzip), process inspection (ps, top, kill), and system monitoring (df,
du, uptime). Correctness is verified by checking final OS state.

\paragraph{ALFWorld \citep{shridhar2020alfworld}:} A text-based embodied household environment aligned with the ALFRED simulator. It translates household manipulation tasks into textual observations and actions while preserving long-horizon planning and partial observability. The benchmark includes six task types: \textit{pick-and-place}, \textit{examine-in-light}, \textit{clean-and-place}, \textit{heat-and-place}, \textit{cool-and-place}, and \textit{pick-two-and-place}. These require agents to locate objects, navigate between receptacles, manipulate object states, and place objects at target locations. The original split provides 3,553 training tasks, 140 validation-seen tasks, and 134 validation-unseen tasks.

\paragraph{AppWorld \citep{trivedi2024appworld}:} AppWorld evaluates
interactive agents in a controllable ecosystem of interconnected
applications. Agents must complete compositional user requests through
multi-step application-API interactions while maintaining consistency across
the resulting workflow. We follow the official evaluation protocol and report
Task Goal Completion (TGC), which measures progress over individual task
goals, and Scenario Goal Completion (SGC), which requires all goals in a
scenario to be completed and therefore provides a stricter measure of
end-to-end success.

\subsection{Controlled MRT Stress Test}
\label{app:mrt_stress_test}

To examine reward contamination in a controlled setting, we replace
$10\%$ of the first-round memory entries with noisy versions. These
entries preserve their original titles but set the key action or
reflection field to \texttt{null}. This construction retains the semantic
cues used for retrieval while removing the actionable content, allowing
the noisy entries to remain retrievable and potentially receive positive
utility updates. Noise is injected only in the first round, after which all methods run
for ten rounds under the same task and interaction budgets. For the
MemRL+UCB variant, we set the UCB exploration coefficient to
$c_{\mathrm{UCB}}=0.2$ and the maximum per-memory exploration bonus to
$b_{\max}=0.3$. We report the round-10 success rate, the average number
of positive utility updates received by noisy entries, and the Final
Noise Ratio, defined as the fraction of noisy entries in the memory pool
at the end of round 10.

The Final Noise Ratio reflects both the persistence of the initially
injected entries and the propagation of noise during subsequent
interactions. In MemRL, we observe that exposure to a retrieved
\texttt{null} entry leads to one additional memory containing the same
\texttt{null} operation. The additional exploration induced by UCB
amplifies this effect, explaining its higher final noise ratio relative
to standard MemRL. In contrast, RoMeRL's replacement mechanism replaces
most noisy contents with higher-quality memories within the first three
rounds, leaving only a small fraction of noisy entries at the end of
training.

\subsection{Q-value stratification and feedback coverage.}
As shown in Figures~\ref{fig:combined}, RoMeRL produces a substantially more informative final Q-value landscape than the MemRL baseline. The association between the learned Q-values and memory-generation provenance is markedly stronger under RoMeRL, with the point-biserial Pearson correlation increasing from $r=0.493$ for the baseline to $r=0.673$ for RoMeRL. Under RoMeRL, the proportion of success-derived memories increases from only $2.3\%$ in the lowest Q-value bin ($0.0$--$0.2$) to $80.1\%$ in the highest bin ($0.9$--$1.0$), demonstrating that the Critic learns a meaningful ranking signal over the complete memory pool. Although the baseline reaches a higher absolute success-derived proportion of $94.7\%$ in its highest-Q bin, its memory pool already contains $77.3\%$ success-derived memories overall, corresponding to only a $1.23\times$ enrichment. In contrast, RoMeRL raises the success-derived proportion from a pool-wide prevalence of $33.3\%$ to $80.1\%$ in the highest-Q bin, yielding a substantially stronger $2.40\times$ enrichment. More importantly, only $5.0\%$ of RoMeRL memories remain at the initial $Q=0.5$, compared with $47.8\%$ for the baseline. This $42.8$-percentage-point reduction indicates that RoMeRL exposes a much larger fraction of the final memory pool to value feedback, thereby avoiding the large uninformative default-Q plateau observed under the baseline.

\subsection{Utility beyond binary success replay.}
The composition of the highest-Q bin further reveals a qualitative difference between the two methods. While the baseline's $0.9$--$1.0$ bin is almost entirely composed of success-derived memories, retaining only $5.3\%$ failure-derived memories, the corresponding RoMeRL bin retains approximately $19.9\%$ failure-derived memories. Consequently, RoMeRL's stronger correlation does not arise from simply copying the binary generation outcome into the Q-value: its high-Q region remains compositionally diverse while still being strongly enriched in success-derived memories. By itself, this compositional evidence shows that the learned Q-value is not a deterministic proxy for the success/failure label. Together with the concrete cases reported in Appendix~\ref{app:case}, where highly valued failure-derived memories contain reusable corrections, diagnostic information, or transferable procedural lessons, the result supports the interpretation that the Critic can recognize utility not reducible to binary episode outcomes. RoMeRL therefore provides a more expressive memory-retention signal than success-only replay: it combines substantially broader Q-update coverage with graded memory ranking, while preserving selected failure-derived experiences that may remain useful for future problem solving.

\begin{figure}[htb]
    \centering
    \begin{subfigure}[b]{0.48\linewidth}
        \centering
        \includegraphics[width=\linewidth]{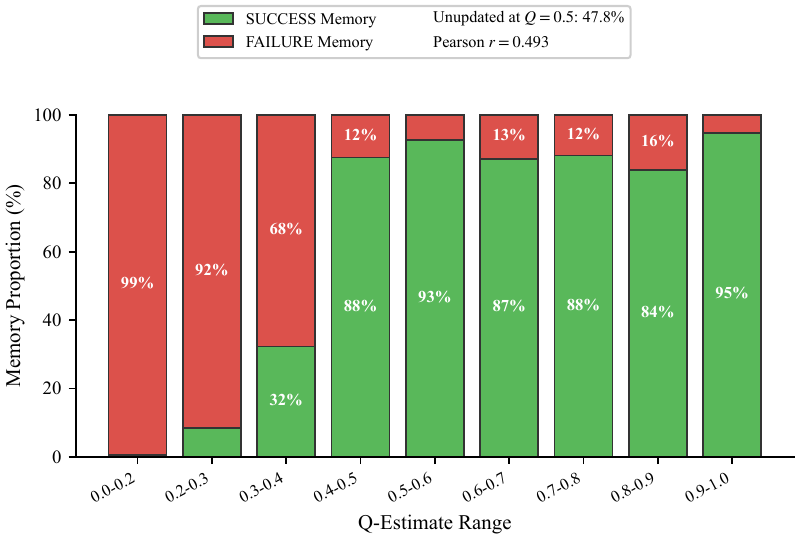}
        \label{fig:q_composition_baseline}
    \end{subfigure}
    \hfill
    \begin{subfigure}[b]{0.48\linewidth}
        \centering
        \includegraphics[width=\linewidth]{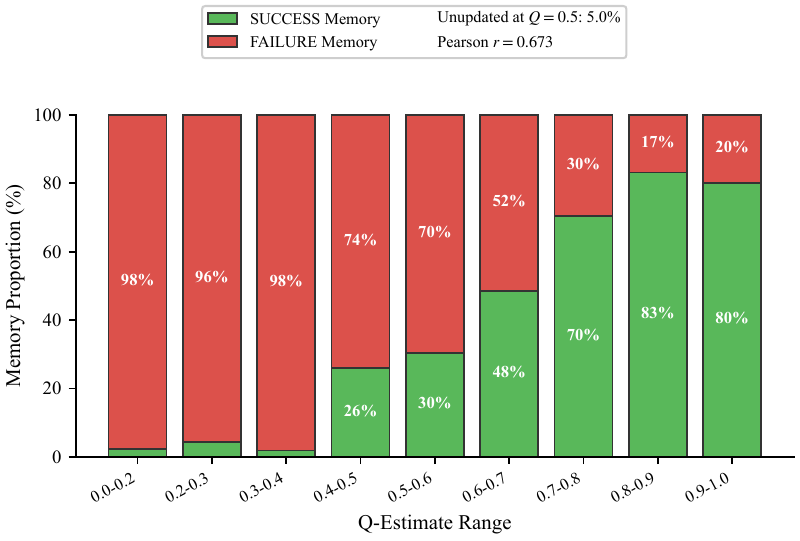}  
    \label{fig:q_composition_romerl}
    \end{subfigure}
    \caption{Final Q-value composition on OS tasks for the MemRL baseline (Left) and  RoMeRL (Right). Each bar shows the proportions of success-derived and failure-derived memories within a Q-value interval; the legends report the fraction of memories remaining at the initial $Q=0.5$ and the Pearson correlation between Q-values and memory-generation outcomes.}
    \label{fig:combined}
\end{figure}

\section{Additional ablation and comparative experiments}
The theoretical analysis in the main text motivates reducing the active
memory dimension, but does not by itself determine the semantic
composition of the reduced-order state. We therefore ground the
consolidated--adaptive distinction in the update, retention, and
replacement mechanisms of existing outcome-driven agent-memory systems.
Although these systems do not usually expose an explicit memory-horizon
parameter, their memory rules implicitly determine how long past evidence
continues to affect the active retrieval state. Such influence may persist
because a trajectory remains directly retrievable, because its feedback
has been accumulated into a learned utility, or because it determines the
representative retained by a semantic coordinate. Consequently, a state
selected from globally accumulated evidence generally has a longer
effective horizon than one selected to reflect recent failures or local
state transitions.

This distinction matches the intended roles of RoMeRL's consolidated and
adaptive coordinates. Consolidated coordinates are selected using evidence
accumulated over the interaction history and preserve their semantic
identity and learned utility when their representatives are updated,
whereas adaptive coordinates track recent failure states or meaningful
local transitions and respond more directly to changes in the current
interaction regime. We therefore use \(h\) as an analytical description
of the temporal influence induced by a memory rule, rather than as an
additional algorithmic hyperparameter; \(h\) characterizes the effective
temporal support of the evidence affecting the active state, not
necessarily the literal age of its current representative. Based on this
abstraction, the following toy study tests the prediction induced by the
incompatibility between the timescales required for rapid adaptation and
long-term retention.

\subsection{Toy Validation of the Consolidated--Adaptive Axis}

Let \(h\) denote this effective memory horizon. Consider an environment
containing short-lived regimes of duration \(L_{\mathrm{shift}}\), together
with a stable regime that disappears temporarily and returns after
\(H_{\mathrm{return}}\) interactions. To incorporate evidence from a
temporary regime before that regime ends, the memory must respond on a
sufficiently short timescale,
\begin{equation}
    h \lesssim L_{\mathrm{shift}}.
\end{equation}
In contrast, preserving the recurring stable solution until it becomes
relevant again requires
\begin{equation}
    h \gtrsim H_{\mathrm{return}}.
\end{equation}
When \(H_{\mathrm{return}}>L_{\mathrm{shift}}\), a single homogeneous
timescale cannot robustly satisfy both requirements. This incompatibility
motivates a short-horizon adaptive role and a long-horizon consolidated
role,
\begin{equation}
    h_A \lesssim L_{\mathrm{shift}},
    \qquad
    h_C \gtrsim H_{\mathrm{return}}.
\end{equation}
The recurrent non-stationary toy experiment is designed to test the
prediction induced by this argument: with the same two-slot memory
capacity, \(C+A\) should provide a better joint adaptation--retention
trade-off than the homogeneous \(C+C\) and \(A+A\) allocations.

We test this prediction using a recurrent non-stationary \(K\)-armed
bandit with \(K=10\). The optimal action follows the regime sequence
\begin{equation}
    a_0
    \rightarrow
    a_1
    \rightarrow
    a_2
    \rightarrow
    a_3
    \rightarrow
    a_0.
    \label{eq:toy-regime-sequence}
\end{equation}
The initial and final \(a_0\) regimes each last \(100\) interactions,
whereas each temporary regime lasts \(25\) interactions. Consequently,
the environment contains both a recurring stable component and a sequence
of short-lived changes. Selecting the optimal action produces a Bernoulli
reward with probability \(0.95\), while any other action succeeds with
probability \(0.05\).

All compared systems maintain exactly two memory slots and use the same
action-selection and exploration policy. They differ only in the temporal
roles assigned to the two slots. The \(C+C\) variant retains the two
actions with the strongest cumulative success evidence. The \(A+A\)
variant retains the two most recently successful distinct actions. The
\(C+A\) variant retains one action according to each rule: its
consolidated slot preserves the action with the strongest long-horizon
evidence, while its adaptive slot tracks the most recently successful
alternative. Because the total capacity is fixed, the comparison isolates
the allocation of memory dynamics rather than the number of stored
representatives.

Moreover, if each slot is restricted to either the consolidated or
adaptive role, the feasible allocations satisfy
\begin{equation}
    n_C+n_A=2,
    \qquad
    n_C,n_A\in\mathbb{Z}_{\geq 0},
\end{equation}
and are therefore exactly
\begin{equation}
    (n_C,n_A)\in\{(2,0),(1,1),(0,2)\},
\end{equation}
corresponding to \(C+C\), \(C+A\), and \(A+A\), respectively. The toy
study thus exhaustively compares the possible two-slot allocations within
this two-timescale hypothesis class.

\begin{figure}[htb]
	\centering
	\includegraphics[width=\linewidth]{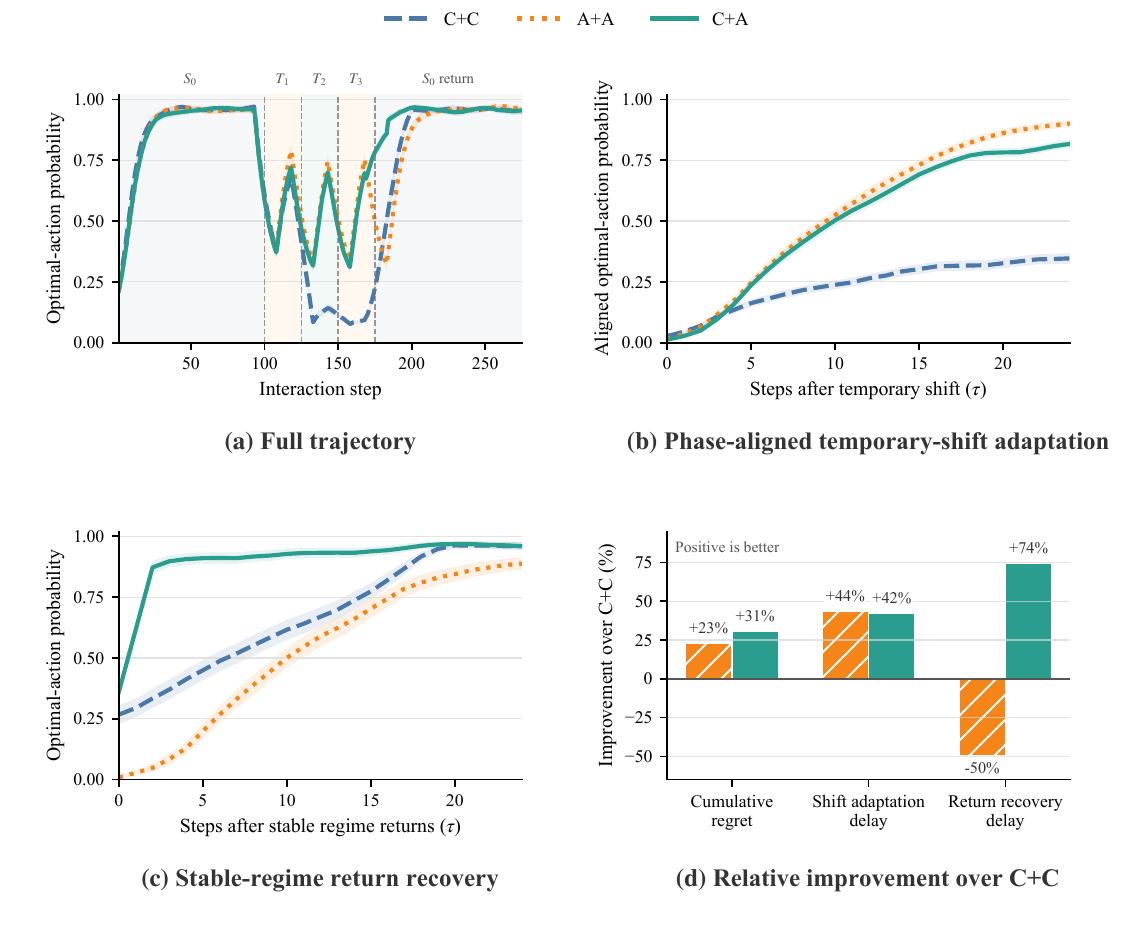}
    \caption{Controlled validation of the consolidated--adaptive distinction under
    equal memory capacity. All methods maintain exactly two slots.}
    \label{fig:toy-ca-stability-adaptation}
\end{figure}

\paragraph{Results.}
Figure~\ref{fig:toy-ca-stability-adaptation} reports the results over
\(500\) random seeds, with the curves showing the probability of selecting
the optimal action. Figure~\ref{fig:toy-ca-stability-adaptation}(a) shows
that \(C+C\) preserves the recurring stable action but adapts slowly during
the temporary regimes, whereas \(A+A\) and \(C+A\) respond substantially
faster. The phase-aligned curves in panel (b) confirm the adaptation
advantage of configurations containing adaptive coordinates. However,
when the original stable regime returns, panel (c) shows that \(A+A\)
recovers slowly because the stable action has been overwritten, while
\(C+A\) recovers almost immediately through its consolidated coordinate.

Panel (d) summarizes this trade-off relative to \(C+C\). The \(A+A\) and
\(C+A\) configurations reduce cumulative pseudo-regret by \(23\%\) and
\(31\%\), respectively, and reduce shift-adaptation delay by \(44\%\) and
\(42\%\). However, \(A+A\) increases return-recovery delay by \(50\%\),
whereas \(C+A\) reduces it by \(74\%\). Thus, within this controlled
two-slot setting, \(C+A\) provides the best joint balance between rapid
adaptation and long-term retention. This result motivates the
consolidated--adaptive distinction as a minimal two-timescale inductive
bias, rather than establishing it as a universally complete memory
factorization. Applying the distinction separately to successful and
failed evidence yields the RoMeRL state
\(\{+,-\}\times\{C,A\}\).

\subsection{Capacity-Matched Comparison with MemRL}
\label{app:capacity_matched_memrl}

To examine whether the improvements of RoMeRL can be explained solely by its bounded memory capacity, we construct a capacity-matched variant of MemRL, denoted as MemRL-4. Specifically, we restrict the memory pool of MemRL to at most four trajectory-indexed entries for each task, matching the maximum number of active coordinates maintained by RoMeRL. Before the capacity limit is reached, newly generated memories are inserted following the original MemRL procedure. Once the pool contains four memories, each newly generated memory replaces the existing entry with the smallest number of utility updates. This feedback-aware replacement rule preferentially removes under-updated historical memories and concentrates subsequent feedback on a bounded set of active entries, providing a stronger comparison than applying a capacity constraint alone.

Except for the memory-capacity constraint and the corresponding replacement rule, MemRL-4 retains the original trajectory-indexed utility representation, retrieval scoring function, and outcome-based utility update of MemRL. We use the same LLM backbone, task sequence, interaction budget, retrieval budget, and evaluation protocol as those used for RoMeRL. Therefore, the comparison between MemRL-4 and RoMeRL evaluates whether the observed gains arise merely from maintaining a small memory pool with feedback-aware replacement, or from the structured reduced-order state introduced by RoMeRL.
\begin{table}[htb]
    \centering
    \footnotesize
    \setlength{\tabcolsep}{3.5pt}
    \renewcommand{\arraystretch}{1.05}
    \caption{
    Capacity-matched results on LifelongAgentBench. Cold-Q is the percentage
    of retrievable memories receiving no utility update, and Feedback Density
    is the number of utility updates per retrievable memory.
    }
    \label{tab:capacity_matched_memrl}
    \resizebox{0.80\linewidth}{!}{%
    \begin{tabular}{l c c c c c}
        \toprule
        \textbf{Method}
        & \textbf{Task}
        & \textbf{Last SR} $\uparrow$
        & \textbf{CSR} $\uparrow$
        & \textbf{Cold-Q (\%)} $\downarrow$
        & \textbf{\shortstack{Feedback\\Density}} $\uparrow$ \\
        \midrule

        \multirow{2}{*}{MemRL-4}
        & OS
        & 0.806
        & 0.814
        & 27.05
        & 11.21 \\

        & DB
        & 0.645
        & 0.940
        & 19.60
        & 8.02 \\

        \midrule

        \multirow{2}{*}{\textit{RoMeRL (ours)}}
        & OS
        & \textbf{0.824}
        & \textbf{0.838}
        & \textbf{9.06}
        & \textbf{29.93} \\

        & DB
        & \textbf{0.680}
        & \textbf{0.952}
        & \textbf{9.73}
        & \textbf{16.07} \\

        \bottomrule
    \end{tabular}%
    }
\end{table}

Table~\ref{tab:capacity_matched_memrl} shows that RoMeRL consistently
outperforms MemRL-4 under the same four-item-per-task capacity. On the OS
task, RoMeRL improves the last-epoch SR from $0.806$ to $0.824$
($+1.8$ percentage points) and the CSR from $0.814$ to $0.838$
($+2.4$ percentage points). On the DB task, the corresponding metrics
increase from $0.645$ to $0.680$ ($+3.5$ percentage points) and from
$0.940$ to $0.952$ ($+1.2$ percentage points), respectively. The
differences are even more pronounced in feedback utilization. Compared
with MemRL-4, RoMeRL reduces the Cold-Q ratio from $27.05\%$ to $9.06\%$
on OS and from $19.60\%$ to $9.73\%$ on DB, corresponding to relative
reductions of $66.5\%$ and $50.4\%$. Meanwhile, its feedback density
increases from $11.21$ to $29.93$ on OS ($2.67\times$) and from $8.02$
to $16.07$ on DB ($2.00\times$). Since MemRL-4 already enforces the
same active-memory capacity and employs a feedback-aware replacement
rule, these results indicate that RoMeRL's improvements cannot be
explained by capacity control alone. Instead, they support the conclusion
that its factorized semantic-coordinate state provides an additional
benefit by concentrating utility feedback on a compact and consistently
updated representation.
\section{Cost and Efficiency Analysis}
\label{app:efficiency_analysis}

\subsection{Token Consumption}

We compare LLM-call cost across the full learning trajectory. Compared with
MemRL, RoMeRL consistently requires fewer runtime calls, as
shown in Figure~\ref{fig:llm_calls_by_section}. This improvement comes from
the factorized per-task memory state. Instead of retrieving and updating
memories from a continuously growing full pool, RoMeRL maintains a
compact active support for each task. As a result, the agent is less likely to
retrieve redundant, stale, or weakly relevant memories, and the prompt contains
more targeted experience for the current task.

This role-guided mechanism reduces cost in two ways. First, RoMeRL avoids generating procedural memories for ineffective trajectories that cannot be promoted or replaced in the role-specific slots, which naturally reduces the number of LLM calls. Second, by providing more reliable success, failure-diagnostic, and recovery information, RoMeRL helps the agent complete tasks in fewer interaction steps, which further reduces the total number of LLM calls. Therefore, the lower cost of RoMeRL is not merely a consequence of using fewer memories, but of replacing full-pool memory accumulation with a compact and role-structured active memory space.

On the LAB OS and DB tasks, RoMeRL achieves lower token consumption and fewer LLM calls than MemRL while maintaining stronger task performance. This indicates that role-guided memory replacement improves not only memory quality, but also the practical efficiency of non-parametric memory learning for autonomous agents.

\begin{figure*}[htb]
    \centering
    \setlength{\abovecaptionskip}{6pt}
    \begin{subfigure}[b]{0.48\textwidth}
        \centering
        \includegraphics[width=\linewidth]{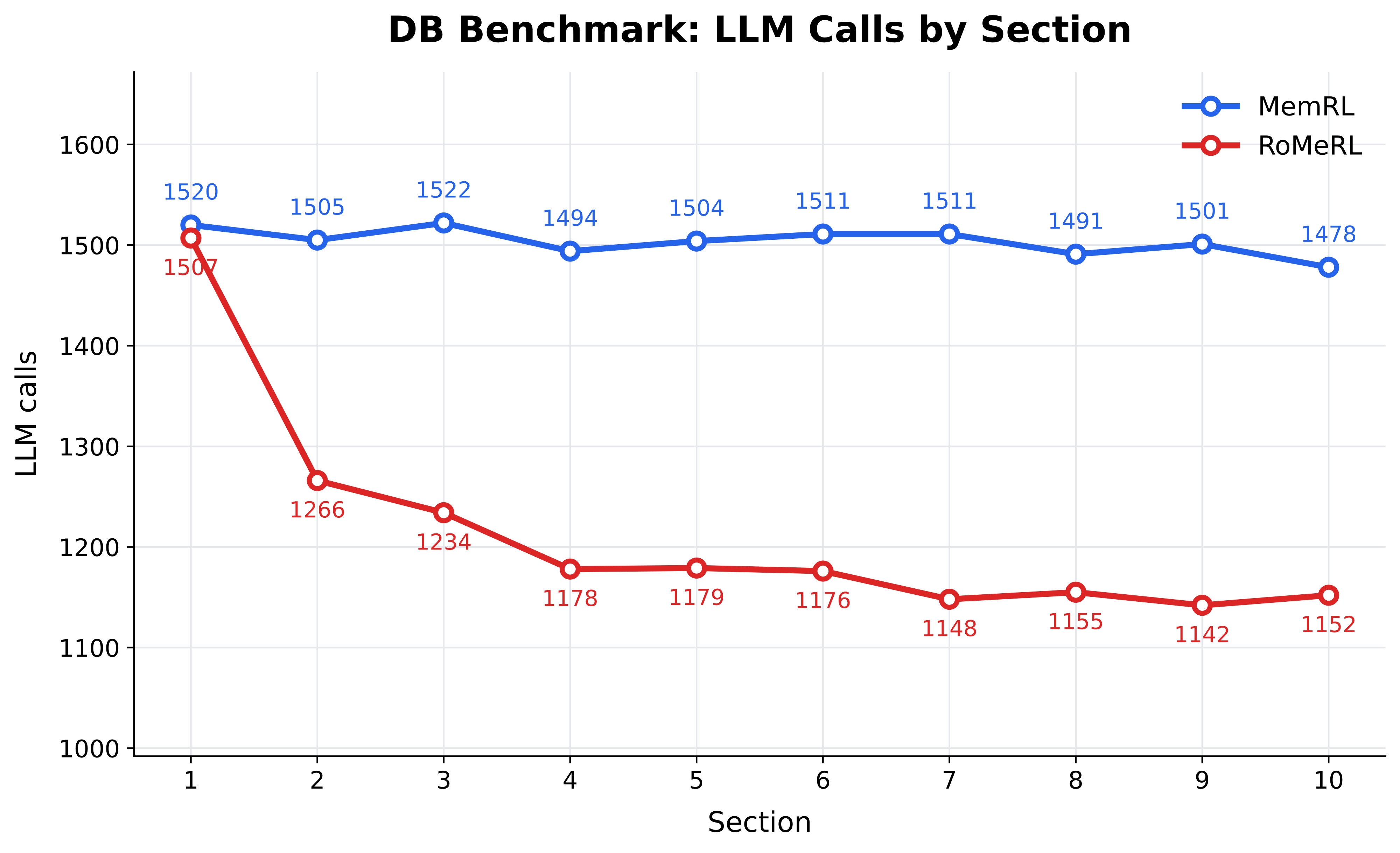}
        
        \label{fig:overall_left}
    \end{subfigure}
    \hfill
    \begin{subfigure}[b]{0.48\textwidth}
        \centering
        \includegraphics[width=\linewidth]{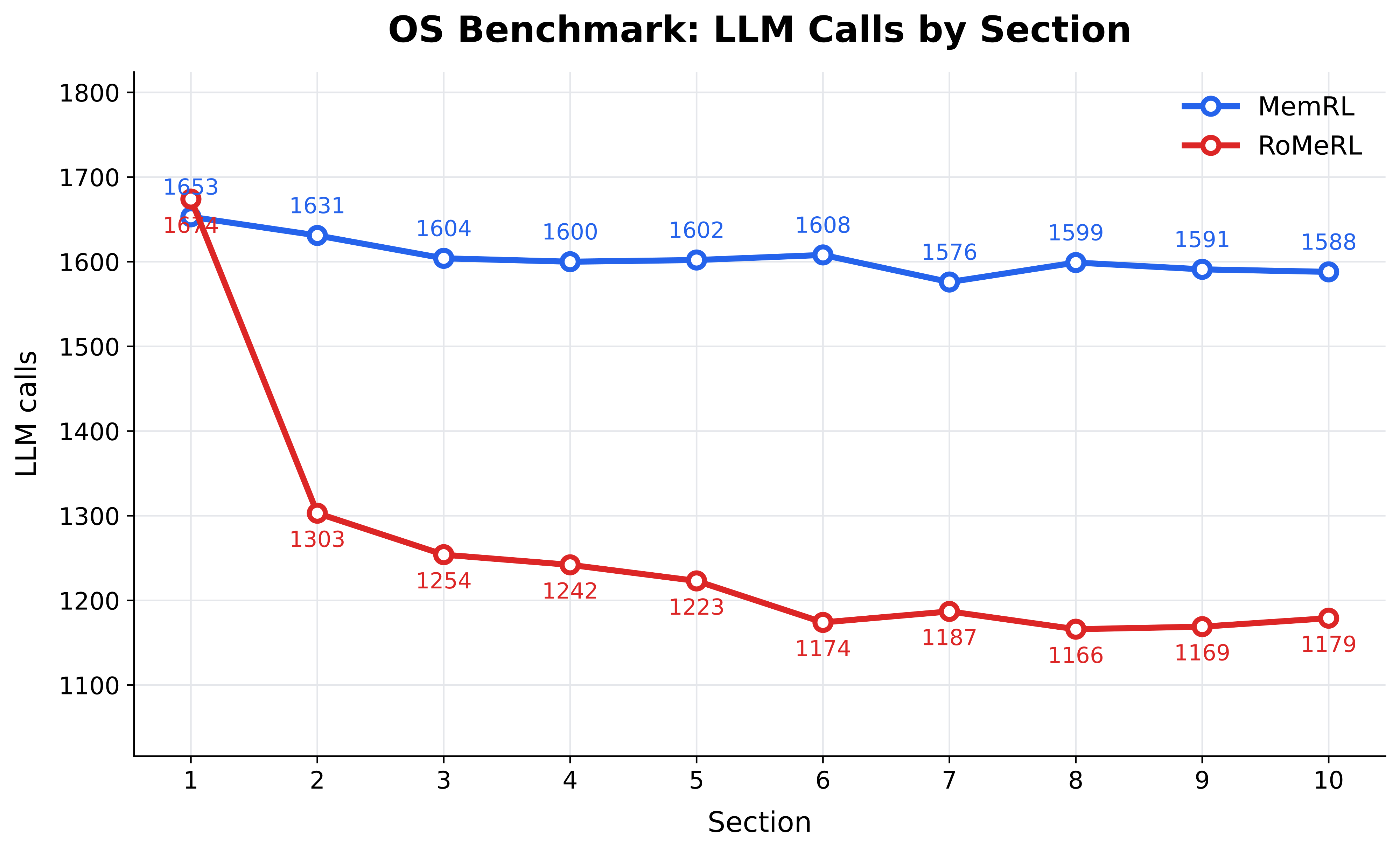}
        
        \label{fig:overall_right}
    \end{subfigure}
    \caption{LLM calls per section on the DB and OS task.
RoMeRL achieves lower and steadily decreasing call counts, while MemRL maintains a consistently high level.}
    \label{fig:llm_calls_by_section}
\end{figure*}

\newpage

\section{Case Study: Role-Based Memory Analysis}
\label{app:case}
This appendix presents qualitative case studies illustrating how RoMeRL's four semantic coordinates are selected and utilized during the 10-epoch OS-interaction run.
CS denotes \emph{Case Study}. Each box below contains the origin task, the retrieved coordinate memory, a brief explanation of its contribution, and the target task where it was retrieved.

\newtcblisting{CodeBox}{%
  enhanced,
  breakable,
  listing only,
  colback=gray!3,
  colframe=black!25,
  boxrule=0.4pt,
  arc=2pt,
  left=6pt,
  right=6pt,
  top=4pt,
  bottom=4pt,
  listing options={%
    basicstyle=\ttfamily\small,
    columns=fullflexible,
    keepspaces=true,
    breaklines=true,
    breakatwhitespace=false,
    showstringspaces=false,
    upquote=true,
    numbers=none,
  },
}

\tcbset{boxsep=1.0mm}

\newtcolorbox{CaseBox}[1]{%
  enhanced,
  breakable,
  colback=gray!5,
  colframe=black!40,
  colbacktitle=blue!15,
  coltitle=black,
  fonttitle=\bfseries\sffamily,
  title={#1},
  attach boxed title to top left={
    xshift=8pt,
    yshift=-8pt
  },
  boxed title style={%
    colback=blue!15,
    colframe=black!40,
    arc=2pt,
    boxrule=0.5pt,
  },
  boxrule=0.5pt,
  arc=3pt,
  left=6pt,
  right=6pt,
  top=10pt,
  bottom=6pt,
}

\begin{CaseBox}{Case Study 1 (\textbf{PCC} --- Stable reproduction of an efficient successful trajectory)}

\textbf{Coordinate:} \textbf{PCC}.

\par\medskip

\textbf{Origin task.}

\begin{CodeBox}
> Create a directory `/target` containing symbolic links named `link1` to `link5` pointing to `/source/file1` to `/source/file5` respectively. Ensure the directory is owned by group `testgroup` with permissions 775.
\end{CodeBox}

\medskip

\textbf{Retrieved success memory.}

\begin{CodeBox}
- #1: [MEMORY TYPE] PROCEDURE
Task: Create a directory '/target' containing symbolic links named 'link1'
to 'link5' pointing to '/source/file1' to '/source/file5' respectively.
Ensure the directory is owned by group 'testgroup' with permissions 775.

SCRIPT:
### High-Level Script

1. Ensure the target group exists.
2. Create the target directory with `mkdir -p`.
3. Set group ownership and permissions with `chgrp` and `chmod`.
4. Create the five symbolic links using absolute source paths.
5. Verify the result with `ls -la /target`.

TRAJECTORY:
assistant: Ensure the group exists, create /target, set its group and mode,
create link1 through link5, and verify the directory.
assistant: Act: bash
assistant: Act: finish
\end{CodeBox}

\par\medskip

\textbf{Explanation.}
This is a standard use of the \textbf{PCC}: the target task and the memory task are identical, and the retained successful trajectory has been compressed into a concise, idempotent, and verifiable operation template. It has a semantic similarity of 1 and maintains a high Q-value after subsequent feedback. The task is again completed in two steps in the tenth evaluation section, indicating stable reproduction of a verified solution rather than renewed exploration.

\par\medskip

\textbf{Target task} (Same task; retrieval rank 1; run outcome: correct).

\begin{CodeBox}
Create a directory '/target' containing symbolic links named 'link1' to
'link5' pointing to '/source/file1' to '/source/file5' respectively. Ensure
the directory is owned by group 'testgroup' with permissions 775.
\end{CodeBox}

\end{CaseBox}
\begin{CaseBox}{Case Study 2 (\textbf{NAC} --- Recent failures as warnings)}

\textbf{Coordinate:} \textbf{NAC}.

\par\medskip

\textbf{Origin task.}

\begin{CodeBox}
Search for all lines containing 'ERROR' in .log files within /var/log/app,
compile a summary report in /var/log/app/summary.txt with error counts per
file and list affected files, then move processed logs to
/var/log/app/archive.
\end{CodeBox}

\medskip

\textbf{Retrieved reflection memory.}

\begin{CodeBox}
- #1: [MEMORY TYPE] FAILURE_REFLECTION
[TASK]
Search for all lines containing 'ERROR' in .log files within /var/log/app,
compile a summary report in /var/log/app/summary.txt with error counts per
file and list affected files, then move processed logs to
/var/log/app/archive.

[REFLECTION]
- ROOT CAUSE: The assistant assumed that /var/log/app existed and was
  writable, and used a non-recursive glob over *.log. The task may also
  contain case-sensitive matching and output-format requirements that were
  not explicitly validated.
- PATTERN TO AVOID: Do not rely on a simple glob or on a successful-looking
  shell transcript. Check the directory, collect the intended files
  explicitly, and verify the exact report format and archive state.
- CORRECT APPROACH: Validate the input directory and permissions, use a
  controlled file list, count ERROR lines per file, write the required
  summary format exactly, move the processed logs, remove temporary files,
  and verify every evaluator-visible artifact.
\end{CodeBox}

\par\medskip

\textbf{Explanation.}
This case study demonstrates both the value and risk of the \textbf{NAC}. It immediately exposes the vulnerabilities of the most recent failed scenario and provides a counterexample for a similar task. However, its content has not been validated, and a bundle-level Q update shared across co-retrieved coordinates may assign a successful target-task outcome to this failure memory.

\par\medskip

\textbf{Target task} (retrieval rank 4; run outcome: correct).

\begin{CodeBox}
Process all error logs in /var/log/app: extract 'ERROR' lines into
errors\_summary.log, write the count to error\_count.txt, create a backup,
set permissions to 644, and move original logs into a backup directory.
\end{CodeBox}

\end{CaseBox}

\begin{CaseBox}{Case Study 3 (\textbf{NCC} --- Preserving diagnostically valuable failure experience)}

\textbf{Coordinate:} \textbf{NCC}.

\par\medskip

\textbf{Origin task.}

\begin{CodeBox}
Create a user 'devuser', create groups 'dev', 'qa', and 'ops', add the
user to all three groups, set the login shell to /bin/zsh, and ensure the
directory '/project' exists with group ownership 'dev' and permissions 770.
\end{CodeBox}

\medskip

\textbf{Retrieved reflection memory.}

\begin{CodeBox}
- #1: [MEMORY TYPE] FAILURE_REFLECTION
[TASK]
Create a user 'devuser', create groups 'dev', 'qa', and 'ops', add the
user to all three groups, set their login shell to /bin/zsh, and ensure the
directory '/project' exists with group ownership 'dev' and permissions 770.

[REFLECTION]
- ROOT CAUSE: The previous solution did not explicitly ensure that /bin/zsh
  was installed and listed in /etc/shells before calling chsh. It also used
  fallback logic that could silently mask user-creation or package-install
  failures.
- PATTERN TO AVOID: Do not assume that installing a shell makes it a valid
  login shell, and do not suppress errors from privileged setup commands.
- CORRECT APPROACH: Verify or install zsh, ensure /bin/zsh is present in
  /etc/shells, create or update the user and all groups idempotently, set the
  shell, configure /project as dev:dev with mode 770, and verify every state.
\end{CodeBox}

\par\medskip

\textbf{Explanation.}
This case demonstrates the transferability of the \textbf{NCC}. Origin task 289 discusses `devuser`, `qa`, `/project`, and `/bin/zsh`, while target task 103 discusses `appuser`, `admins`, `/var/app/app.log`, and `/bin/bash`. Although the task entities differ, they share the operational structure of user/group creation, shell modification, permission setting, and explicit verification.

\par\medskip

\textbf{Target task} (retrieval rank 1; run outcome: correct).

\begin{CodeBox}
Create groups 'dev', 'ops', and 'admins'; add user 'appuser' to all three
groups, set their login shell to /bin/bash, enforce password expiry in 30
days, create /var/app/app.log, make the log file group-readable, and make
the home directory group-accessible.
\end{CodeBox}

\end{CaseBox}
\begin{CaseBox}{Case Study 4 (\textbf{PAC} --- Retaining the first successful recovery after failure)}

\textbf{Coordinate:} \textbf{PAC}.

\par\medskip

\textbf{Origin task.}

\begin{CodeBox}
Create a group 'devteam', add users 'alice' and 'bob' to it, set 'alice' as
the group administrator, and configure a shared directory '/devteam_shared'
accessible only by the group.
\end{CodeBox}

\medskip

\textbf{Retrieved sucess memory.}

\begin{CodeBox}
- #1: [MEMORY TYPE] PROCEDURE  [COORDINATES: PCC + PAC]
Task: Create a group 'devteam', add users 'alice' and 'bob' to it, set 'alice' as the group administrator, and configure a shared directory '/devteam_shared' accessible only by the group.

SCRIPT:
### High-Level Script for Setting Up a Group-Administered Shared Directory

1. Ensure the group and both target users exist.
2. Add alice and bob to devteam and set alice as administrator with `gpasswd -A alice devteam`.
3. Create /devteam_shared, set its group to devteam, and use mode 770.
4. Verify membership, administrator assignment, group ownership, and mode.

TRAJECTORY:
assistant: Create the group and users, configure membership and the group administrator, create the 770 shared directory, and verify all conditions.
assistant: Act: bash
assistant: Act: finish
\end{CodeBox}

\par\medskip

\textbf{Explanation.}
The stored memory is a successful procedure. It occupies the \textbf{PAC}
because this trajectory was the first success after task 114's earlier
failure; the same trajectory may also occupy the PCC when it is the most
efficient successful representative. It is not a separate failure-reflection
record.

\par\medskip

\textbf{Target task} (Same task; retrieval rank 1; run outcome: correct).

\begin{CodeBox}
Create a group 'devteam', add users 'alice' and 'bob' to it, set 'alice' as
the group administrator, and configure a shared directory '/devteam_shared'
accessible only by the group.
\end{CodeBox}
\end{CaseBox}
\newpage
\section{Prompt Details}
We provide the exact prompt strings and message templates used by our RoMeRL
implementation for LifelongAgentBench and ALFWorld. For clarity, we separate
the prompts that summarize experiences into memories from those used at task
time for generation and inference. AppWorld follows the benchmark's native
application-API interface and official TGC/SGC evaluators.
\tcbset{boxsep=1.0mm}
\newtcolorbox{PromptBox}[1]{%
  enhanced,
  breakable,
  colback=white,
  colframe=black!55,
  colbacktitle=black!6,
  coltitle=black,
  fonttitle=\bfseries,
  title={#1},
  boxrule=0.5pt,
  arc=2pt,
  left=6pt,right=6pt,top=6pt,bottom=6pt,
}
\newtcblisting{PromptCode}{%
  listing engine=listings,
  listing only,
  breakable,
  colback=black!1,
  colframe=black!35,
  boxrule=0.4pt,
  arc=2pt,
  left=6pt,right=6pt,top=4pt,bottom=4pt,
  listing options={
    basicstyle=\ttfamily\small,
    columns=fullflexible,
    keepspaces=true,
    breaklines=true,
    breakatwhitespace=false,
    showstringspaces=false,
    numbers=none
  },
}
\subsection{Experience Summarization Prompts}
\label{app:memrl-prompts:experience-summarization}
\begin{PromptBox}{ALFWorld: Experience Summarization Prompts}
\textbf{Trajectory serialization (stored as the episode trajectory).}
\begin{PromptCode}
The full ALFWorld dialogue history is stored as the trajectory:
  List[{"role": ..., "content": ...}, ...]
When interpolated into the summarization prompts, it is treated as a string representation.
\end{PromptCode}
\medskip
\textbf{High-level script generation prompt.}
\begin{PromptCode}
Analyze the following detailed task trajectory and create a concise,
high-level script that captures the essential steps and decision points.

The script should be:
1. Generic enough to apply to similar tasks
2. Specific enough to provide useful guidance
3. 3-5 high-level steps maximum
4. Focus on the strategy and key decisions, not detailed actions

Trajectory:
{trajectory}

High-level script:
\end{PromptCode}
\medskip
\textbf{Failure reflection prompt.}
\begin{PromptCode}
Task: {task_description}

Failed trajectory:
{failed_trajectory}

This task failed. Analyze what went wrong and suggest improvements for future similar tasks.
Focus on:
1. Incorrect assumptions
2. Steps to improve
3. What to avoid next time

Provide a brief reflection:
\end{PromptCode}
\medskip
\textbf{Stored memory content templates.}
\begin{PromptCode}
# Successful memory
Task: {task_description}

SCRIPT:
{script}

TRAJECTORY:
{trajectory}

# Failure memory
TASK REFLECTION:
Task: {task_description}

What went wrong:
{reflection}

Failed approach:
{failed_trajectory}
\end{PromptCode}
\end{PromptBox}

\begin{PromptBox}{LLB (LifelongAgentBench): Experience Summarization Prompts}
\textbf{Trajectory serialization (stored as the episode trajectory).}
\begin{PromptCode}
{role_1}: {content_1}
{role_2}: {content_2}
...
\end{PromptCode}
\medskip
\textbf{High-level script generation prompt.}
\begin{PromptCode}
Analyze the following detailed task trajectory and create a concise,
high-level script that captures the essential steps and decision points.

The script should be:
1. Generic enough to apply to similar tasks
2. Specific enough to provide useful guidance
3. 3-5 high-level steps maximum
4. Focus on the strategy and key decisions, not detailed actions

Trajectory:
{trajectory}

High-level script:
\end{PromptCode}
\medskip
\textbf{Failure reflection prompt.}
\begin{PromptCode}
Task: {task_description}

Failed trajectory:
{failed_trajectory}

This task failed. Analyze what went wrong and suggest improvements for future similar tasks.
Focus on:
1. Incorrect assumptions
2. Steps to improve
3. What to avoid next time

Provide a brief reflection:
\end{PromptCode}
\medskip
\textbf{Stored memory content templates.}
\begin{PromptCode}
# Successful memory
Task: {task_description}

SCRIPT:
{script}

TRAJECTORY:
{trajectory}

# Failure memory
TASK REFLECTION:
Task: {task_description}

What went wrong:
{reflection}

Failed approach:
{failed_trajectory}
\end{PromptCode}
\end{PromptBox}
\subsection{Generation and Inference Prompts}
\label{app:memrl-prompts:generation-inference}
\begin{PromptBox}{ALFWorld: Generation and Inference Prompts}
\textbf{Base system prompt (ReAct format + action space).}
\begin{PromptCode}
Interact with a household to solve a task. Imagine you are an intelligent agent in a household environment and your target is to perform actions to complete the task goal. At the beginning of your interactions, you will be given the detailed description of the current environment and your goal to accomplish.
For each of your turn, you will be given the observation of the last turn. You should first think about the current condition and plan for your future actions, and then output your action in this turn. Your output must strictly follow this format:"Thought: your thoughts.\nAction: your next action".

The available actions are:
1. go to {recep}
2. take {obj} from {recep}
3. move {obj} to {recep}
4. open {recep}
5. close {recep}
6. use {obj}
7. clean {obj} with {recep}
8. heat {obj} with {recep}
9. cool {obj} with {recep}
where {obj} and {recep} correspond to objects and receptacles.
After your each turn, the environment will give you immediate feedback based on which you plan your next few steps. if the envrionment output "Nothing happened", that means the previous action is invalid and you should try more options.

Your response should use the following format:

Thought: <your thoughts>
Action: <your next action>
\end{PromptCode}
\medskip
\textbf{Retrieved memory injection (system message).}
\begin{PromptCode}
In addition to the example, you have the following memories from your own past experiences. Use them to help you if they are relevant:

--- SUCCESSFUL MEMORIES (Examples to follow) ---
{formatted_success_memories_joined}

--- FAILED MEMORIES (Examples to avoid or learn from) ---
{formatted_failed_memories_joined}
\end{PromptCode}
\medskip
\textbf{Current task prompt (user message).}
\begin{PromptCode}
Now, it's your turn to solve a new task.
{task_description}
\end{PromptCode}
\medskip
\textbf{Per-step observation prompt (user message).}
\begin{PromptCode}
Observation: {observation}
\end{PromptCode}
\medskip
\textbf{Message ordering (high level).}
\begin{PromptCode}
1) system: base ALFWorld system prompt
2) user/assistant: selected few-shot example dialogue (sequence of messages)
3) system: optional retrieved memory context
4) user: new task prompt
5) loop: append user Observation: ..., model replies with Thought/Action
\end{PromptCode}
\end{PromptBox}
\medskip

\medskip

\begin{PromptBox}{LLB (LifelongAgentBench): Generation and Inference Prompts}
\textbf{Base system prompt.}
\begin{PromptCode}
You are an execution-focused AI agent solving database and operating-system tasks.

You may receive a [Retrieved Memory Context] block with past experiences from similar problems.
These are **references for learning**, not guaranteed solutions:
- [MEMORY TYPE] SUCCESS_PROCEDURE: A successful approach from a similar task learn the pattern.
- [MEMORY TYPE] FAILURE_REFLECTION: A failed attempt with lessons avoid similar mistakes.

Use the memories as inspiration, but always analyze your current task independently and
adapt your approach based on its specific requirements.
\end{PromptCode}
\medskip
\textbf{Strict output constraint (DB tasks).}
\begin{PromptCode}
STRICT OUTPUT FORMAT (LLB:DB, do not violate):
1) After your reasoning, include exactly ONE action line:
   - Action: Operation
   - Action: Answer
2) If Action: Operation, put exactly ONE SQL statement in the FIRST fenced code block using ```sql, on a single line. Do not add any extra text after that block.
3) If Action: Answer, include `Final Answer: ...` on the next line and do not add extra text after that.
\end{PromptCode}
\medskip
\textbf{Strict output constraint (OS tasks).}
\begin{PromptCode}
STRICT OUTPUT FORMAT (LLB:OS, do not violate):
1) After your reasoning, include exactly ONE action line:
   - Act: bash
   - Act: finish
2) If Act: bash, the next lines MUST be a ```bash fenced code block with your Bash commands. Do not include any other code blocks.
3) If Act: finish, it must be the last line (no code blocks, no extra text).
4) Do NOT use `Action:` in OS tasks (use `Act:` only).
\end{PromptCode}
\medskip
\textbf{Retrieved memory injection block.}
\begin{PromptCode}
[Retrieved Memory Context]

=== SUCCESSFUL EXPERIENCES (Learn from these) ===
[SUCCESS 1] [TYPE: {mem_type}]
{content}

=== FAILED EXPERIENCES (Avoid these mistakes) ===
[FAILURE 1] [TYPE: {mem_type}]
{content}
\end{PromptCode}
\medskip
\textbf{Prompt assembly ordering (system prompt).}
\begin{PromptCode}
1) base system prompt
2) optional [Retrieved Memory Context] ...
3) strict output format block appended at the end (task-aligned)
\end{PromptCode}
\end{PromptBox}

\end{document}